%% file: ordinal.tex
\documentclass[11pt]{article}
\usepackage{graphicx}
\usepackage{subfigure}
\usepackage{caption}
\usepackage{hyperref}
\usepackage[margin=1in]{geometry}

\usepackage{graphicx}

\usepackage{hyperref}       
\usepackage{url}            

\usepackage{amsmath,amssymb}
\usepackage{amsthm}    

\usepackage{bm}
\usepackage{mathtools}
\mathtoolsset{showonlyrefs}
\usepackage{enumitem}
\theoremstyle{plain}
\newtheorem{thm}{Theorem}[section]
\newtheorem{lem}{Lemma}

\newtheorem{assumption}{Assumption}

\theoremstyle{definition}
\newtheorem{defn}{Definition}
\newtheorem{cor}{Corollary}
\newtheorem{example}{Example}
\newtheorem{rmk}{Remark}

\allowdisplaybreaks
\usepackage{dsfont}

\usepackage{color}
\usepackage{algorithm}
\usepackage{algorithmic}
\usepackage{natbib}
\usepackage{forloop}

\newcommand*{\KeepStyleUnderBrace}[1]{
  \mathop{%
    \mathchoice
    {\underbrace{\displaystyle#1}}%
    {\underbrace{\textstyle#1}}%
    {\underbrace{\scriptstyle#1}}%
    {\underbrace{\scriptscriptstyle#1}}%
  }\limits
}
\usepackage{makecell}
\input macros.tex

\usepackage{amssymb}

\usepackage{pifont}
\newcommand{\cmark}{\ding{51}}%
\newcommand{\xmark}{\ding{55}}%

\title{Tensor denoising and completion based on ordinal observations}
\date{}
\author{%
Chanwoo Lee \\
University of Wisconsin -- Madison\\
\texttt{chanwoo.lee@wisc.edu} \\
\and
Miaoyan Wang \\
University of Wisconsin -- Madison\\
\texttt{miaoyan.wang@wisc.edu} \\
}

\begin{document}
\maketitle
\begin{abstract}
Higher-order tensors arise frequently in applications such as neuroimaging, recommendation system, and social network analysis. We consider the problem of low-rank tensor estimation from possibly incomplete, ordinal-valued observations. Two related problems are studied, one on tensor denoising and the other on tensor completion. We propose a multi-linear cumulative link model, develop a rank-constrained M-estimator, and obtain theoretical accuracy guarantees. Our mean squared error bound enjoys a faster convergence rate than previous results, and we show that the proposed estimator is minimax optimal under the class of low-rank models. Furthermore, the procedure developed serves as an efficient completion method which guarantees consistent recovery of an order-$K$ $(d,\ldots,d)$-dimensional low-rank tensor using only $\tilde \tO(Kd)$ noisy, quantized observations. We demonstrate the outperformance of our approach over previous methods on the tasks of clustering and collaborative filtering.
\end{abstract}

\section{Introduction}
Multidimensional arrays, a.k.a.\ tensors, arise in a variety of applications including recommendation systems~\citep{baltrunas2011incarmusic}, social networks~\citep{nickel2011three}, genomics~\citep{wang2019three}, and neuroimaging~\citep{zhou2013tensor}. There is a growing need to develop general methods that can address two main problems for analyzing these noisy, high-dimensional datasets. The first problem is tensor denoising which aims to recover a signal tensor from its noisy entries~\citep{hong2020generalized,wang2019multiway}. The second problem is tensor completion which examines the minimum number of entries needed for a consistent recovery~\citep{ghadermarzy2018learning,montanari2018spectral}. Low-rankness is often imposed to the signal tensor, thereby efficiently reducing the intrinsic dimension in both problems.

A number of low-rank tensor estimation methods have been proposed~\citep{hong2020generalized,wang2017tensor}, revitalizing classical methods such as CANDECOMP/PARAFAC (CP) decomposition~\citep{hitchcock1927expression} and Tucker decomposition~\citep{tucker1966some}. These tensor methods treat the entries as continuous-valued. In many cases, however, we encounter datasets of which the entries are qualitative. For example, the Netflix problem records the ratings of users on movies over time. Each data entry is a rating on a nominal scale \{{\it very like, like, neutral, dislike, very dislike}\}. Another example is in the signal processing, where the digits are frequently rounded or truncated so that only integer values are available. The qualitative observations take values in a limited set of categories, making the learning problem harder compared to  continuous observations.

Ordinal entries are categorical variables with an ordering among the categories; for example, {\it very like} $\prec$ {\it like} $\prec$ {\it neutral} $\prec \cdots$. The analyses of tensors with the ordinal entries are mainly complicated by two key properties needed for a reasonable model. First, the model should be invariant under a reversal of categories, say, from the Netflix example, {\it very like} $\succ$ {\it like} $\succ$ {\it neutral} $\succ \cdots$, but not under arbitrary label permutations. Second, the parameter interpretations should be consistent under merging or splitting of contiguous categories. The classical continuous tensor model~\citep{kolda2009tensor, ghadermarzy2019near} fails in the first aspect, whereas the binary tensor model~\citep{ghadermarzy2018learning} lacks the second property. An appropriate model for ordinal tensors has yet to be studied.

\begin{table}[ht]
\resizebox{\textwidth}{!}{
\begin{tabular}{l|ccc}
&\cite{bhaskar2016probabilistic}&\cite{ghadermarzy2018learning} &This paper\\
\hline 
Higher-order tensors ($K\geq 3$) & \xmark&\cmark& \cmark\\
\hline
Multi-level categories ($L\geq 3$)& \cmark& \xmark&\cmark\\
 \hline
 Error rate for tensor denoising&$d^{-1}$ for $K=2$& $d^{-(K-1)/2}$&$d^{-(K-1)}$\\
\hline
Optimality guarantee under low-rank models& unknown & \xmark & \cmark\\
\hline
Sample complexity for tensor completion&$d^K$& $Kd$&$Kd$
\end{tabular}
}
\caption{Comparison with previous work. For ease of presentation, we summarize the error rate and sample complexity assuming equal tensor dimension in all modes. $K$: tensor order; $L$: number of ordinal levels; $d$: dimension at each mode. }\label{tab:compare}
\end{table}

{\bf Our contributions.} We establish the recovery theory for signal tensors and quantization operators simultaneously from a limited number of highly discrete entries. Our main contributions are summarized in Table~\ref{tab:compare}. We propose a cumulative link model for higher-order tensors, develop a rank-constrained M-estimator, and obtain theoretical accuracy guarantees. The mean squared error bound is established, and we show that the obtained bound has minimax optimal rate in high dimensions under the low-rank model. Furthermore, our proposal guarantees consistent recovery of an order-$K$ $(d,\ldots,d)$-dimensional low-rank tensor using only $\tilde \tO(Kd)$ noisy, quantized observations.

Our work is connected to non-Gaussian tensor decomposition. Existing work focuses exclusively on univariate observations such as binary- or continuous-valued entries~\citep{wang2020learning,ghadermarzy2018learning}.  The problem of ordinal/quantized tensor is fundamentally more challenging than previously well-studied tensors for two reasons: (a) the entries do not belong to exponential family distribution, and (b) the observation contains much less information, because neither the underlying signal nor the quantization operator is known. The limited information makes the statistical recovery notably hard.

We address the challenge by proposing a cumulative link model that enjoys palindromic invariance~\citep{mccullagh1980regression}. A distinctive non-monotonic, phase-transition pattern is demonstrated, as we show in Section~\ref{sec:experiment}. We prove that the recovery from quantized tensors achieves equally good information-theoretical convergence as the continuous tensors. These results fills the gap between classical and non-classical (ordinal) observations, thereby greatly enriching the tensor model literature. 

From algorithm perspective, we address the challenge using the (non-convex) alternating algorithm. Earlier work has proposed an approximate (convex) algorithm for binary tensor completion~\citep{ghadermarzy2018learning}. Unlike matrix problems, convex-relaxation for low-rank tensors suffers from both computational intractability~\citep{hillar2013most} and statistical suboptimality. We improve the error bound from $\tO(d^{-(K-1)/2})$ in~\citet{ghadermarzy2018learning} to $\tO(d^{-(K-1)})$ and numerically compare the two approaches. 

We also highlight the challenge associated with higher-order tensors. Matrix completion has been proposed for binary observations~\citep{cai2013max,davenport2014,bhaskar20151} and for ordinal observations~\citep{bhaskar2016probabilistic}. We show that, applying existing matrix methods to higher-order tensors results in suboptimal estimates. A full exploitation of the higher-order structure is needed; this is another challenge we address in this paper.

\section{Preliminaries}
Let $\tY\in\mathbb{R}^{d_1\times \cdots \times d_K}$ denote an order-$K$ $(d_1,\ldots,d_K)$-dimensional tensor. We use $y_\omega$ to denote the tensor entry indexed by $\omega$, where $\omega\in[d_1]\times\cdots\times[d_K]$.  The Frobenius norm of $\tY$ is defined as $\FnormSize{}{\tY}=\sum_\omega y^2_\omega$ and the infinity norm is defined as $\mnormSize{}{\tY}=\max_{\omega}|y_\omega|$. We use $\tY_{(k)}$ to denote the unfolded matrix of size $d_k$-by-$\prod_{i\neq k}d_i$, obtained by reshaping the tensor along the mode $k\in[K]$. The Tucker rank of $\tY$ is defined as a length-$K$ vector $\mr = (r_1,\ldots,r_K)$, where $r_k$ is the rank of matrix $\tY_{(k)}$ for $k \in[K]$.  

We use lower-case letters ($a, b, \ldots$) for scalars/vectors, upper-case boldface letters ($\mA, \mB,  \ldots$) for matrices, and calligraphy letters ($\tA,\tB,\ldots$) for tensors of order three or greater.  An event $A$ is said to occur ``with very high probability'' if $\mathbb{P}(A)$ tends to 1 faster than any polynomial of tensor dimension $d_{\min}=\min\{d_1,\ldots,d_K\} \to\infty$. The indicator function of an event $A$ is denoted as $\mathds{1}\{A\}$. For ease of notation, we allow basic arithmetic operators (e.g., $\leq, +, -$) to be applied to pairs of tensors in an element-wise manner. We use the shorthand $[n]$ to denote $\{1,\ldots,n\}$ for $n \in N_{+}$.

\section{Model formulation and motivation}
\subsection{Observation model}
Let $\tY$ denote an order-$K$ $(d_1,\ldots,d_K)$-dimensional data tensor. Suppose the entries of $\tY$ are ordinal-valued, and the observation space consists of $L$ ordered levels, denoted by $[L]=\{1,\ldots,L\}$. We propose a cumulative link model for the ordinal tensor $\tY=\entry{y_\omega}\in[L]^{d_1\times \cdots\times d_K}$. Specifically, assume the entries $y_\omega$ are (conditionally) independently distributed with cumulative probabilities,
\begin{equation}\label{eq:model}
\mathbb{P}(y_\omega\leq \ell)=f(b_\ell-\theta_\omega),\ \text{for all}\ \ell\in[L-1],
\end{equation}
where $\mb=(b_1,\ldots,b_{L-1})$ is a set of unknown scalars satisfying $b_1<\cdots <b_{L-1}$, $\Theta=\entry{\theta_\omega}\in\mathbb{R}^{d_1\times \cdots \times d_K}$ is a continuous-valued parameter tensor satisfying certain low-dimensional structure (to be specified later), and $f(\cdot)\colon \mathbb{R}\mapsto[0,1]$ is a known, strictly increasing function. We refer to $\mb$ as the cut-off points and $f$ the link function.

The formulation~\eqref{eq:model} imposes an additive model to the transformed probability of cumulative categories. This modeling choice is to respect the ordering structure among the categories. For example, if we choose the inverse link $f^{-1}(x)=\log {x\over 1-x}$ to be the log odds, then the model~\eqref{eq:model} implies a linear spacing between the proportional odds,
\begin{equation}\label{eq:logodd}
\log {\mathbb{P}(y_{\omega}\leq \ell) \over \mathbb{P}(y_\omega >  \ell) } - \log {\mathbb{P}(y_{\omega}\leq {\ell-1}) \over \mathbb{P}(y_\omega >  {\ell-1}) } = b_\ell-b_{\ell-1},
\end{equation}
for all tensor entries $y_\omega$. When there are only two categories in the observation space (e.g., for binary tensors), the cumulative model~\eqref{eq:model} is equivalent to the usual binomial link model. In general, however, when the number of categories $L\geq 3$, the proportional odds assumption~\eqref{eq:logodd} is more parsimonious, in that, the ordered categories can be envisaged as contiguous intervals on the continuous scale, where the points of division are exactly $b_1<\cdots <b_{L-1}$. This interpretation will be made explicit in the next section.

\subsection{Latent-variable interpretation}\label{sec:latent}
The ordinal tensor model~\eqref{eq:model} with certain types of link $f$ has the equivalent representation as an $L$-level quantization model. 

Specifically, the entries of tensor $\tY=\entry{y_\omega}$ are modeled from the following generative process, 
\begin{align}\label{eq:quantization}
y_\omega&=
\begin{cases}
1,& \text{if $y^*_\omega\in(-\infty, b_1]$},\\
2,& \text{if $y^*_\omega\in(b_1, b_2]$},\\
\vdots  &\vdots\\
L,& \text{if $y^*_\omega\in(b_{L-1}, \infty)$},\\
\end{cases}
\end{align}
for all $\omega\in[d_1]\times \cdots \times [d_k]$. Here, $\tY^*=\entry{y^*_\omega}$ is a latent continuous-valued tensor following an additive noise model,
\begin{equation}\label{eq:latent}
\KeepStyleUnderBrace{\tY^*}_{\text{latent continuous-valued tensor}}=\KeepStyleUnderBrace{\Theta}_{\text{signal tensor}}+\KeepStyleUnderBrace{\tE}_{\text{i.i.d.\ noise}},
\end{equation}
where $\tE=\entry{\varepsilon_\omega}\in\mathbb{R}^{d_1\times \cdots \times d_K}$ is a noise tensor with independent and identically distributed (i.i.d.) entries according to distribution $\mathbb{P}(\varepsilon)$. From the viewpoint of~\eqref{eq:latent}, the parameter tensor $\Theta$ can be interpreted as the latent signal tensor prior to contamination and quantization.

The equivalence between the latent-variable model~\eqref{eq:quantization} and the cumulative link model~\eqref{eq:model} is established if the link $f$ is chosen to be the cumulative distribution function of noise $\varepsilon$, i.e., $f(\theta)=\mathbb{P}(\varepsilon\leq \theta)$. We describe two common choices of link $f$, or equivalently, the distribution of $\varepsilon$.

\begin{example}[Logistic model] The logistic model is characterized by~\eqref{eq:model} with $f(\theta)=(1+e^{-\theta/\sigma})^{-1}$, where $\sigma>0$ is the scale parameter. Equivalently, the noise $\varepsilon_\omega$ in~\eqref{eq:quantization} follows i.i.d.\ logistic distribution with scale parameter $\sigma$.
\end{example}
\begin{example}[Probit model] The probit model is characterized by~\eqref{eq:model} with
$f(\theta)=\mathbb{P}(z\leq \theta/\sigma)$, where $z\sim N(0,1)$. Equivalently, the noise $\varepsilon_\omega$ in~\eqref{eq:quantization} follows i.i.d.\ $N(0,\sigma^2)$.
\end{example}
Other link functions are also possible, such as Laplace, Cauchy, etc~\citep{mccullagh1980regression}. These latent variable models share the property that the ordered categories can be thought of as contiguous intervals on some continuous scale. We should point out that, although the latent-variable interpretation is incisive, our estimation procedure does not refer to the existence of $\tY^*$. Therefore, our model~\eqref{eq:model} is general and still valid in the absence of quantization process. More generally, we make the following assumptions about the link $f$.

\begin{assumption}\label{ass:link}
The link function $f$ is assumed to satisfy:
\begin{enumerate}[label=\textup{(\roman*)}]
\item The function $f(\theta)$ is strictly increasing and twice-differentiable in $\theta\in \mathbb{R}$.
\item The derivative $f'(\theta)$ is strictly log-concave and symmetric with respect to $\theta=0$.
\end{enumerate}
\end{assumption}

\subsection{Problem 1: Tensor denoising}~\label{sec:denoising}
The first question we aim to address is tensor denoising:\\

(P1) Given the quantization process induced by $f$ and the cut-off points $\mb$, how accurately can we estimate the latent signal tensor $\Theta$ from the ordinal observation $\tY$?\\

Clearly, the problem (P1) cannot be solved uniformly for all possible $\Theta$ with no assumptions. We focus on a class of ``low-rank'' and ``flat'' signal tensors, which is a plausible assumption in practical applications~\citep{zhou2013tensor,bhaskar20151}. Specifically, we consider the parameter space,
\begin{equation}\label{eq:space}
\tP=\left\{\Theta\in\mathbb{R}^{d_1\times \cdots \times d_K} \colon \text{rank}(\Theta)\leq \mr,\ \mnormSize{}{\Theta}\leq \alpha\right\},
\end{equation}
where $\mr=(r_1,\ldots,r_K)$ denotes the Tucker rank of $\Theta$.

The parameter tensor of our interest satisfies two constraints. The first is that $\Theta$ is a low-rank tensor, with $r_k=\tO(1)$ as $d_{\min}\to \infty$ for all $k\in[K]$. Equivalently, $\Theta$ admits the Tucker decomposition:
\begin{equation}\label{eq:Tucker}
\Theta=\tC\times_1\mM_1\times_1 \cdots \times_K \mM_K,
\end{equation}
where $\tC\in\mathbb{R}^{r_1\times \cdots \times r_K}$ is a core tensor, $\mM_k\in\mathbb{R}^{d_k\times r_k}$ are factor matrices with orthogonal columns, and $\times_k$ denotes the tensor-by-matrix multiplication~\citep{kolda2009tensor}. The Tucker low-rankness is popularly imposed in tensor analysis, and the rank determines the tradeoff between model complexity and model flexibility. Note that, unlike matrices, there are various notions of tensor low-rankness, such as CP rank~\citep{hitchcock1927expression} and train rank~\citep{oseledets2011tensor}. Some notions of low-rankness may lead to mathematically ill-posed optimization; for example, the best low CP-rank tensor approximation may not exist~\citep{de2008tensor}. We choose Tucker representation for well-posedness of optimization and easy interpretation.

The second constraint is that the entries of $\Theta$ are uniformly bounded in magnitude by a constant $\alpha \in \mathbb{R}_{+}$. In view of~\eqref{eq:latent}, we refer to $\alpha$ as the signal level. The boundedness assumption is a technical condition that avoids the degeneracy in probability estimation with ordinal observations.

\subsection{Problem 2: Tensor completion}
Motivated by applications in collaborative filtering, we also consider a more general setup when only a subset of tensor entries $y_\omega$ are observed. Let $\Omega\subset[d_1]\times \cdots\times[d_K]$ denote the set of observed indices. The second question we aim to address is stated as follows:\\

(P2) Given an incomplete set of ordinal observations $\{y_{\omega}\}_{\omega\in\Omega}$, how many sampled entries do we need to consistently recover $\Theta$ based on the model~\eqref{eq:model}?\\

The answer to (P2) depends on the choice of $\Omega$. We consider a general model on $\Omega$ that allows both uniform and non-uniform sampling. Specifically, let $\Pi=\{\pi_{i_1,\ldots,i_K}\}$ denote a predefine probability distribution over the index set such that $\sum_{\omega\in[d_1]\times \cdots \times [d_K]} \pi_\omega =1$. We assume that each index in $\Omega$ is drawn with replacement using distribution $\Pi$. This sampling model relaxes the uniform sampling in literature and is arguably a better fit in applications.

We consider the same parameter space~\eqref{eq:space} for the completion problem. In addition to the reasons mentioned in Section~\ref{sec:denoising}, the entrywise bound assumption also serves as the incoherence requirement for completion. In classical matrix completion, the incoherence is often imposed on the singular vectors. This assumption is recently relaxed for ``flat'' matrices with bounded magnitude~\citep{negahban2011estimation,cai2013max,bhaskar20151}. We adopt the same assumption for higher-order tensors.

\section{Rank-constrained M-estimator}\label{sec:theory}
We present a general treatment to both problems mentioned above. With a little abuse of notation, we use $\Omega$ to denote either the full index set $\Omega=[d_1]\times \cdots \times [d_K]$ (for the tensor denoising) or a random subset induced from the sampling distribution $\Pi$ (for the tensor completion). Define $b_0=-\infty$, $b_L=\infty$, $f(-\infty)=0$ and $f(\infty)=1$. The log-likelihood associated with the observed entries is
\begin{equation}\label{eq:objective}
\logl(\Theta, \mb)=\sum_{\omega\in\Omega}\sum_{\ell\in[L]} \Big\{\mathds{1}\{y_\omega=\ell\} \log \big[f(b_\ell-\theta_\omega)- f(b_{\ell-1}-\theta_\omega)\big]\Big\}.
\end{equation}
We propose a rank-constrained maximum likelihood estimator (a.k.a.\ M-estimator) for $\Theta$,
\begin{align}\label{eq:estimator}
\hat \Theta&=\argmax_{\Theta\in \tP}\logl(\Theta, \mb),\ \text{where}\notag \\
\tP&=\left\{\Theta\in\mathbb{R}^{d_1\times \cdots \times d_K} \colon \text{rank}(\Theta)\leq \mr,\ \mnormSize{}{\Theta}\leq \alpha\right\}.
 \end{align}
In practice, the cut-off points $\mb$ are unknown and should be jointly estimated with $\Theta$. We will present the theory and algorithm with known $\mb$ in the main paper. The adaptation for unknown $\mb$ is addressed in Section~\ref{sec:algorithm} and the Supplement.

We define a few key quantities that will be used in our theory. Let $g_\ell=f(\theta+b_\ell)-f(\theta+b_{\ell-1})$ for all $\ell\in[L]$, and
\begin{equation}\label{eq:regular}
A_\alpha = \min_{\ell\in[L], |\theta|\leq \alpha} g_\ell(\theta),\quad  U_\alpha=\max_{\ell\in[L], |\theta|\leq \alpha } {|\dot{g}_\ell(\theta)|\over g_\ell(\theta)},\quad L_\alpha = \min_{\ell\in[L], |\theta|\leq \alpha }\left[{\dot{g}_\ell^2 (\theta)\over g_\ell^2(\theta)} -{\ddot{g}_\ell(\theta)\over g_\ell(\theta)}\right],
\end{equation}
where $\dot{g}(\theta)=dg(\theta)/d\theta$, and $\alpha$ is the entrywise bound of $\Theta$. In view of equation~\eqref{eq:latent}, these quantities characterize the geometry including flatness and convexity of the latent noise distribution. Under Assumption~\ref{ass:link}, all these quantities are strictly positive and independent of tensor dimension.

\subsection{Estimation error for tensor denoising}\label{sec:denosing}
For the tensor denoising problem, we assume that the full set of tensor entries are observed. We assess the estimation accuracy using the mean squared error (MSE):
\[
\text{MSE}(\hat \Theta, \trueT)={1\over \prod_k d_k}\FnormSize{}{\Theta-\trueT}^2.
\]
The next theorem establishes the upper bound for the MSE of the proposed $\hat \Theta$ in~\eqref{eq:estimator}.

\begin{thm}[Statistical convergence] \label{thm:rate}
Consider an ordinal tensor $\tY\in[L]^{d_1\times\dots\times d_K}$ generated from model~\eqref{eq:model}, with the link function $f$ and the true coefficient tensor $\trueT\in\tP$. Define $r_{\max}=\max_k r_k$. Then, with very high probability, the estimator in~\eqref{eq:estimator} satisfies
\begin{equation}\label{eq:rate}
\mathrm{MSE}(\hat \Theta, \trueT) \leq \min\left( 4\alpha^2,\ {c_1  U^2_\alpha r_{\max}^{K-1}  \over  L^2_\alpha } {\sum_kd_k\over  \prod_k d_k} \right),
\end{equation}
where $c_1 >0$ is a constant that depends only on $K$.
\end{thm}
Theorem~\ref{thm:rate} establishes the statistical convergence for the estimator~\eqref{eq:estimator}. In fact, the proof of this theorem (see the Supplement) shows that the same statistical rate holds, not only for the global optimizer~\eqref{eq:estimator}, but also for any local optimizer $\check \Theta$ in the level set $\{\check\Theta\in\tP\colon \logl(\check\Theta)\geq \logl(\trueT)\}$. This suggests that the local optimality itself is not necessarily a severe concern in our context, as long as the convergent objective is large enough. In Section ~\ref{sec:algorithm}, we perform empirical studies to assess the algorithmic stability.

To gain insight into the bound~\eqref{eq:rate}, we consider a special setting with equal dimension in all modes, i.e., $d_1=\cdots=d_K=d$. In such a case, our bound \eqref{eq:rate} reduces to
\begin{equation}\label{eq:ours}
\text{MSE}(\hat \Theta, \trueT) \asymp d^{-(K-1)}, \quad \text{as}\ d\to \infty.
\end{equation}
Hence, our estimator achieves consistency with polynomial convergence rate. We compare the bound with existing literature. In the special case $L=2$, \citet{ghadermarzy2018learning} proposed a max-norm constrained estimator $\tilde{\Theta}$ with $\text{MSE}(\tilde{\Theta}, \trueT) \asymp  d^{-(K-1)/2}$. In contrast, our estimator converges at a rate of $d^{-(K-1)}$, which is substantially faster than theirs. This provides a positive answer to the open question posed in~\citet{ghadermarzy2018learning} whether the square root in the bound is removable. The improvement stems from the fact that we have used the exact low-rankness of $\Theta$, whereas the surrogate rank measure employed in~\citet{ghadermarzy2018learning} is scale-sensitive.

Our bound also generalizes the previous results on ordinal matrices. The convergence rate for rank-constrained matrix estimation is $\tO(1/\sqrt{d})$~\citep{bhaskar2016probabilistic}, which fits into our special case when $K=2$. Furthermore, our result~\eqref{eq:rate} reveals that the convergence becomes favorable as the order of data tensor increases. Intuitively, the sample size for analyzing a data tensor is the number of entries, $\prod_k d_k$, and the number of free parameters is roughly on the order of $\sum_{k}d_k$, assuming $r_{\max}=\tO(1)$. A higher tensor order implies higher effective sample size per parameter, thus achieving a faster convergence rate in high dimensions.

A similar conclusion is obtained for the prediction error, measured in Kullback-Leibler (KL) divergence, between the categorical distributions in the observation space.
\begin{cor}[Prediction error]~\label{cor:prediction}
Assume the same set-up as in Theorem~\ref{thm:rate}. Let $\mathbb{P}_{\tY}$ and $\hat{\mathbb{P}}_{\tY}$ denote the distributions generating the $L$-level ordinal tensor $\tY$, given the true parameter $\Theta$ and its estimator $\hat \Theta$, respectively. Assume $L\geq 2$. Then, with very high probability,
\begin{equation}\label{eq:KLrate}
\text{KL}(\mathbb{P}_{\tY} || \hat{\mathbb{P}}_{\tY}) \leq  {c_1 U^2_\alpha r_{\max}^{K-1}  \over L^2_\alpha } {(4L-6)\dot{f}^2(0)  \over A_\alpha} {\sum_kd_k\over  \prod_k d_k},
\end{equation}
where $c_1 >0$ is the same constant as in Theorem~\ref{thm:rate}.
\end{cor}

We next show the statistical optimality of our estimator $\hat \Theta$. The result is based on the information theory and applies to all estimators in $\tP$, including but not limited to $\hat \Theta$ in~\eqref{eq:estimator}.

\begin{thm}[Minimax lower bound]\label{thm:minimax}
Assume the same set-up as in Theorem~\ref{thm:rate}, and $d_{\max}=\max_k d_k \geq 8$. Let $\inf_{\hat \Theta}$ denote the infimum over all estimators $\hat \Theta\in\tP$ based on the ordinal tensor observation $\tY\in[L]^{d_1\times \cdots \times d_K}$. Then, under the model~\eqref{eq:model},
\begin{equation}\label{eq:lower}
\inf_{\hat \Theta }\sup_{\trueT\in\tP} \mathbb{P}\Big\{ \textup{MSE}(\hat \Theta, \trueT) \geq c\min\left( \alpha^2, \ { Cr_{\max}d_{\max} \over \prod_k d_k} \right) \Big\} \geq {1\over 8},
\end{equation}
where $C=C(\alpha, L, f,\mb)>0$ and $c>0$ are constants independent of tensor dimension and the rank.
\end{thm}
We see that the lower bound matches the upper bound in~\eqref{eq:rate} on the polynomial order of tensor dimension. Therefore, our estimator~\eqref{eq:estimator} is rate-optimal.

\subsection{Sample complexity for tensor completion}
We now consider the tensor completion problem, when only a subset of entries $\Omega$ are observed. We consider a general sampling procedure induced by $\Pi$. The recovery accuracy is assessed by the weighted squared error,
\begin{equation}\label{eq:weighted}
\PiFnormSize{}{\Theta-\hat \Theta}^2\stackrel{\text{def}}{=}
{1\over |\Omega|}\mathbb{E}_{\Omega\sim \Pi}\FnormSize{}{\Theta-\hat \Theta}^2=\sum_{\omega\in[d_1]\times \cdots \times [d_K]} \pi_{\omega}(\Theta_{\omega}-\hat \Theta_{\omega})^2.
\end{equation}

Note that the recovery error depends on the distribution $\Pi$. In particular, tensor entries with higher sampling probabilities have more influence on the recovery accuracy, compared to the ones with lower sampling probabilities.
\begin{rmk} If we assume each entry is sampled with some strictly positive probability; i.e.,\ there exits a constant $\mu> 0$ such that
\[
\ \pi_\omega\geq {1\over \mu \prod_k d_k},\quad \text{for all}\ \omega\in[d_1]\times\cdots \times [d_K],
\]
then the error in~\eqref{eq:weighted} provides an upper bound for MSE:
\[
\PiFnormSize{}{\Theta-\hat \Theta}^2 \geq {\FnormSize{}{\Theta-\hat \Theta}^2\over \mu \prod_kd_k}={1\over \mu}\text{MSE}(\hat \Theta, \trueT).
\]
The equality is attained under uniform sampling with $\mu=1$.
\end{rmk}

\begin{thm}~\label{thm:completion}
Assume the same set-up as in Theorem~\ref{thm:rate}. Suppose that we observe a subset of tensor entries $\{y_\omega\}_{\omega\in\Omega}$, where $\Omega$ is chosen at random with replacement according to a probability distribution $\Pi$. Let $\hat \Theta$ be the solution to~\eqref{eq:estimator}, and assume $r_{\max}=\tO(1)$. Then, with very high probability,
\[
\PiFnormSize{}{\Theta-\hat \Theta}^2 \to 0,\quad \text{ as }\quad {|\Omega|\over \sum_k d_k} \ \to \infty.
\]
\end{thm}
Theorem~\ref{thm:completion} shows that our estimator achieves consistent recovery using as few as $\tilde{\tO}(Kd)$ noisy, quantized observations from an order-$K$ $(d,\ldots,d)$-dimensional tensor. Note that $\tilde \tO(Kd)$ roughly matches the degree of freedom for an order-$K$ tensor of fixed rank $\mr$, suggesting the optimality of our sample requirement. This sample complexity substantially improves over earlier result $\tO(d^{\lceil K/2\rceil})$ based on square matricization~\citep{mu2014square}, or $\tO(d^{K/2})$ based on tensor nuclear-norm regularization~\citep{yuan2016tensor}. Existing methods that achieve $\tilde \tO(Kd)$ sample complexity require either a deterministic cross sampling design~\citep{zhang2019cross} or univariate measurements~\citep{ghadermarzy2018learning}. Our method extends the conclusions to multi-level measurements under a broader class of sampling schemes.

\section{Numerical Implementation}\label{sec:algorithm}
We describe the algorithm to seek the optimizer of~\eqref{eq:objective}. In practice, the cut-off points $\mb$ are often unknown, so we choose to maximize $\logl$ jointly over $(\Theta,\mb)\in\tP\times \tB$ (see the Supplement for details). The objective $\logl$ is concave in $(\Theta,\mb)$ whenever $f'$ is log-concave. However, the feasible set $\tP$ is non-convex, which makes the optimization~\eqref{eq:objective} a non-convex problem. We employ the alternating optimization approach by utilizing the Tucker representation of $\Theta$. Specifically, based on~\eqref{eq:Tucker} and~\eqref{eq:objective}, the objective function consists of $K+2$ blocks of variables, one for the cut-off points $\mb$, one for the core tensor $\tC$, and $K$ for the factor matrices $\mM_k$'s. The optimization is a simple convex problem if any $K+1$ out of the $K+2$ blocks are fixed. We update one block at a time while holding others fixed, and we alternate the optimization throughout the iteration. The convergence is guaranteed whenever $\logl$ is bounded from above, since the alternating procedure monotonically increases the objective. The Algorithm~\ref{alg} gives the full description.

\begin{algorithm}[t]
        \caption{Ordinal tensor decomposition}\label{alg}
        \begin{algorithmic}[]
            \STATE{\bfseries Input:}  \text{ Ordinal data tensor }
            $\mathcal{Y}\in [L]^{d_1\times\cdots\times d_K}$, rank $\mr\in \mathbb{N}_{+}^{K}$, entry-wise bound $\alpha\in \mathbb{R_+}$.
            \STATE{\bfseries Output:} $(\hat\Theta,\hat{\mb}) =  \argmax_{(\Theta,\mb)\in \mathcal{P}\times\mathcal{B}}  \mathcal{L}_{\mathcal{Y},\Omega}(\Theta,\mb).$
            \STATE Random initialization of core tensor $\mathcal{C}^{(0)}$, factor matrices $ \{\mM_k^{(0)}\}$, and cut-off points $\mb^{(0)}$.
            \FOR{$t = 1,2,\cdots,$}
            \FOR{$k = 1,2,\cdots,K$}
            \STATE{Update $\mM^{(t+1)}_k$ while fixing other blocks:}
            \STATE $\mM_k^{(t+1)}\gets\argmax_{\mM_k\in\mathbb{R}^{d_k\times r_K}}\mathcal{L}_{\mathcal{Y},\Omega}(\mM_k)$,
            \newline
             s.t. $\mnormSize{}{\Theta^{(t+1)}}\leq \alpha$, where $\Theta^{(t+1)}$ is the parameter tensor based on the current block estimates.
            \ENDFOR
            \STATE {Update $\mathcal{C}^{(t+1)}$ while fixing other blocks:}
              \newline$\mathcal{C}^{(t+1)} \gets \argmax_{\mathcal{C}\in \mathbb{R}^{r_1\times\cdots\times r_K}}\mathcal{L}_{\mathcal{Y},\Omega}(\mathcal{C})$, s.t. $\mnormSize{}{\Theta^{(t+1)}}\leq \alpha.$
               \STATE {Update $\Theta^{(t+1)}$ based on the current block estimates:}
                   \newline
                   $\Theta^{(t+1)} \gets \mathcal{C}^{(t+1)}\times_1\mM_1^{(t+1)}\cdots\times_K\mM_K^{(t+1)}.$
            \STATE {Update $\mb^{(t+1)}$ while fixing $\Theta^{(t+1)}$:}
              \newline
              $\mb^{(t+1)} \gets \argmax_{\mb\in \tB}\mathcal{L}_{\mathcal{Y},\Omega}\big(\Theta^{(t+1)},\mb\big).$
            \ENDFOR
            \STATE \textbf{return}
            $(\hat \Theta,\hat \mb)$
        \end{algorithmic}
 \end{algorithm}
We comment on two implementation details before concluding this section. First, the problem~\eqref{eq:estimator} is non-convex, so Algorithm 1 usually has no theoretical guarantee on global optimality. Nevertheless, as shown in Section~\ref{sec:denosing}, the desired rate holds not only for the global optimizer, but also for the local optimizer with $\logl(\hat \Theta)\geq \logl(\trueT)$. In practice, we find the convergence point $\hat \Theta$ upon random initialization is often satisfactory, in that the corresponding objective $\logl(\hat \Theta)$ is close to and actually slightly larger than the objective evaluated at the true parameter $\logl(\trueT)$. Figure~\ref{fig:stability} shows the trajectory of the objective function that is output in the default setting of Algorithm 1, with the input tensor generated from probit model~\eqref{eq:model} with $d_1=d_2=d_3=d$ and $r_1=r_2=r_3=r$. The dashed line is the objective value at the true parameter $\logl(\trueT)$. We find that the algorithm generally converges quickly to a desirable value in reasonable number of steps. The actual running time per iteration is shown in the plot legend.

\begin{figure}[ht]
\centering
\includegraphics[width=13cm]{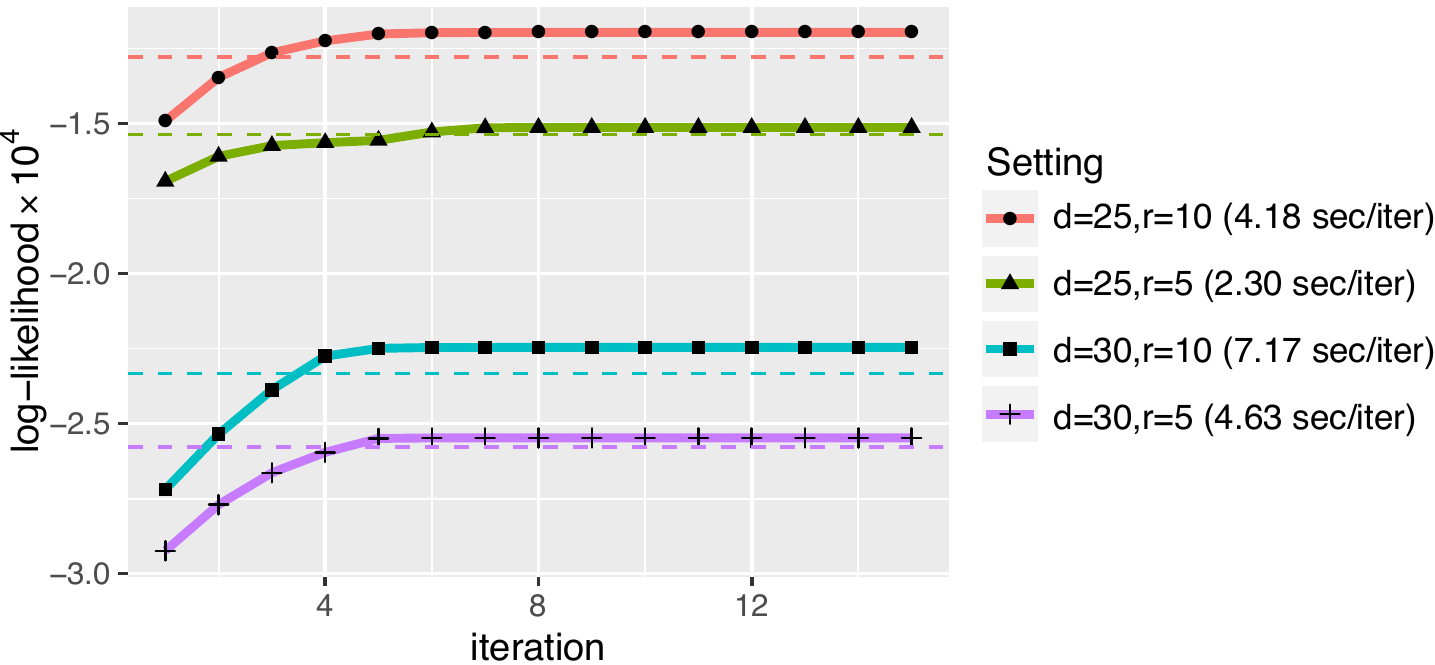}
\caption{Trajectory of objective function with various $d$ and $r$.}\label{fig:stability}
\end{figure}

Second, the algorithm takes the rank $\mr$ as an input. In practice, the rank $\mr$ is hardly known and needs to be estimated from the data. We use Bayesian information criterion (BIC) and choose the rank that minimizes BIC; i.e.,
\begin{align}\label{eq:BIC}
\hat \mr&=\argmin_{\mr\in\mathbb{N}^K_{+}} \text{BIC}(\mr)\\
&=\argmin_{\mr\in\mathbb{N}^K_{+}}\{-2\tL_{\tY}(\hat \Theta(\mr),\hat \mb(\mr))+p_e(\mr)\log (\prod_k d_k)\},
\end{align}
where $\hat \Theta(\mr), \hat\mb(\mr)$ are the estimates given the rank $\mr$, and $p_e(\mr)\stackrel{\text{def}}{=}\sum_k (d_k-r_k)r_k+\prod_k r_k$ is the effective number of parameters in the model. We select $\hat \mr$ that minimizes BIC through a grid search. The choice of BIC is intended to balance between the goodness-of-fit for the data and the degrees of freedom in the population model.

\section{Experiments}\label{sec:experiment}
In this section, we evaluate the empirical performance of our methods\footnote{Software package: \url{https://CRAN.R-project.org/package=tensorordinal}}. We investigate both the complete and the incomplete settings, and we compare the recovery accuracy with other tensor-based methods. Unless otherwise stated, the ordinal data tensors are generated from model~\eqref{eq:model} using standard probit link $f$. We consider the setting with $K=3$, $d_1=d_2=d_3=d$, and $r_1=r_2=r_3=r$. The parameter tensors are simulated based on~\eqref{eq:Tucker}, where the core tensor entries are i.i.d.\ drawn from $N(0,1)$, and the factors $\mM_k$ are uniformly sampled (with respect to Haar measure) from matrices with orthonormal columns. We set the cut-off points $b_\ell=f^{-1}(\ell/L)$ for $\ell\in[L]$, such that $f(b_\ell)$ are evenly spaced from 0 to 1. In each simulation study, we report the summary statistics across $n_{\text{sim}}$ = 30 replications.

\subsection{Finite-sample performance}\label{sec:simulation}
The first experiment examines the performance under complete observations. We assess the empirical relationship between the MSE and various aspects of model complexity, such as dimension $d$, rank $r$, and signal level $\alpha=\mnormSize{}{\Theta}$. Figure~\ref{fig:finite}a plots the estimation error versus the tensor dimension $d$ for three different ranks $r\in\{3,5,8\}$. The decay in the error appears to behave on the order of $d^{-2}$, which is consistent with our theoretical results~\eqref{eq:rate}. We find that a higher rank leads to a larger error, as reflected by the upward shift of the curve as $r$ increases. Indeed, a higher rank implies the higher number of parameters to estimate, thus increasing the difficulty of the estimation. Figure~\ref{fig:finite}b shows the estimation error versus the signal level under $d=20$. Interestingly, a larger estimation error is observed when the signal is either too small or too large. The non-monotonic behavior may seem surprising, but this is an intrinsic feature in the estimation with ordinal data. In view of the latent-variable interpretation (see Section~\ref{sec:latent}), estimation from ordinal observation can be interpreted as an inverse problem of quantization. Therefore, the estimation error diverges in the absence of noise $\tE$, because it is impossible to distinguish two different signal tensors, e.g., $\Theta_1=\ma_1\otimes \ma_2\otimes \ma_3$ and $\Theta_2=\text{sign}(\ma_1)\otimes \text{sign}(\ma_2)\otimes \text{sign}(\ma_3)$, from the quantized observations. This phenomenon~\citep{davenport2014,sur2019modern} is clearly contrary to the classical continuous-valued tensor problem.

The second experiment investigates the incomplete observations. We consider $L$-level tensors with $d=20$, $\alpha=10$ and choose a subset of tensor entries via uniform sampling. Figure~\ref{fig:finite}c shows the estimation error of $\hat \Theta$ versus the fraction of observation $\rho=|\Omega|/d^K$. As expected, the error reduces with increased $\rho$ or decreased $r$. Figure~\ref{fig:finite}d evaluates the impact of ordinal levels $L$ to estimation accuracy, under the setting $\rho=0.5$. An improved performance is observed as $L$ grows, especially from binary observations ($L=2$) to multi-level ordinal observations ($L\geq 3$). The result showcases the benefit of multi-level observations compared to binary observations.

\begin{figure}[http]
\includegraphics[width=\textwidth]{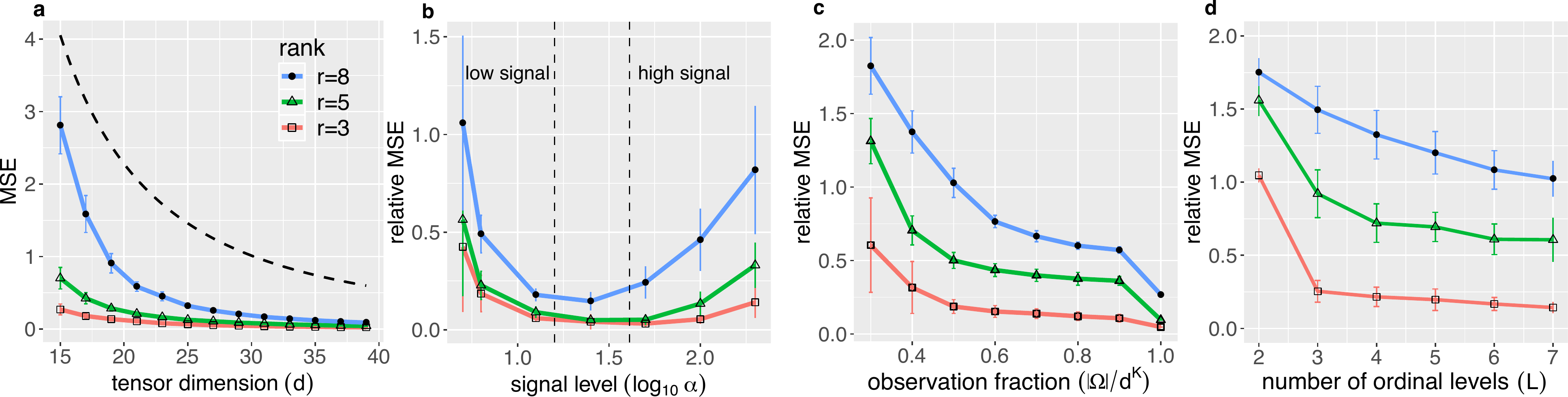}
\caption{Empirical relationship between (relative) MSE versus (a) dimension $d$, (b) signal level $\alpha$, (c) observation fraction $\rho$, and (d) number of ordinal levels $L$. In panels (b)-(d), we plot the relative MSE $=\FnormSize{}{\hat \Theta-\trueT}^2/\FnormSize{}{\trueT}^2$ for better visualization.}
\label{fig:finite}
\end{figure}

\subsection{Comparison with alternative methods}
Next, we compare our ordinal tensor method ({\bf Ordinal-T}) with three popular low-rank methods.
\begin{itemize}
\item Continuous tensor decomposition ({\bf Continuous-T})~\citep{acar2010scalable} is a low-rank approximation method based on classical Tucker model.
\item One-bit tensor completion ({\bf 1bit-T})~\citep{ghadermarzy2018learning} is a max-norm penalized tensor learning method based on partial binary observations.
\item Ordinal matrix completion ({\bf Ordinal-M})~\citep{bhaskar2016probabilistic} is a rank-constrained matrix estimation method based on noisy, quantized observations.
\end{itemize}

We apply each of the above methods to $L$-level ordinal tensors $\tY$ generated from model~\eqref{eq:model}. The {\bf Continuous-T} is applied directly to $\tY$ by treating the $L$ levels as continuous observations. The {\bf Ordinal-M} is applied to the 1-mode unfolding matrix $\tY_{(1)}$. The {\bf 1bit-T} is applied to $\tY$ in two ways. The first approach, denoted {\bf 1bit-sign-T}, follows from~\citet{ghadermarzy2018learning} that transforms
$\tY$ to a binary tensor, by taking the entrywise sign of the mean-adjusted tensor, $\tY- |\Omega|^{-1}\sum_\omega y_\omega$. The second approach, denoted {\bf 1bit-category-T}, transforms the order-3 ordinal tensor $\tY$ to an order-4 binary tensor, $\tY^{\sharp}=\entry{y^{\sharp}_{ijkl}}$, via dummy variable encoding; i.e., $y^{\sharp}_{ijk\ell}=\mathds{1}\{y_{ijk}=\ell\}$ for $\ell\in[L-1]$. We evaluate the methods by their capabilities in predicting the most likely labels, $y_\omega^{\text{mode}}=\arg\max_\ell\mathbb{P}(y_\omega=\ell)$. Two performance metrics are considered: mean absolute deviation (MAD) $=d^{-K}\sum_\omega |y_\omega^{\text{mode}}-\hat y_\omega^{\text{mode}}|$, and misclassification rate (MCR) $=d^{-K}\sum_\omega\mathds{1}\{y_\omega^{\text{mode}}\neq\text{round}(\hat y_\omega^{\text{mode}})\}$,
where $\text{round}(\cdot)$ denotes the nearest integer. Note that MAD penalizes the large deviation more heavily than MCR.

Figure~\ref{fig:compare} compares the prediction accuracy under the setting $\alpha=10$, $d=20$, and $r=5$. We find that our method outperforms the others in both MAD and MCR. In particular, methods built on multi-level observations ({\bf Ordinal-T}, {\bf Ordinal-M}, {\bf 1bit-category-T}) exhibit stable MCR over $\rho$ and $L$, whereas the others two methods ({\bf Continuous-T}, {\bf 1bit-sign-T}) generally fail except for $L=2$ (Figures~\ref{fig:compare}a-b). This observation highlights the benefits of multi-level modeling in the classification task. Although {\bf 1bit-category-T} and our method {\bf Ordinal-T} behave similarly for binary tensors ($L=2$), the improvement of our method is substantial as $L$ increases (Figures~\ref{fig:compare}a and~\ref{fig:compare}c). One possible reason is that our method incorporates the intrinsic ordering among the $L$ levels via proportional odds assumption~\eqref{eq:logodd}, whereas {\bf 1bit-category-T} ignores the ordinal structure and dependence among the induced binary entries. 

Figures~\ref{fig:compare}c-d assess the prediction accuracy with sample size. We see a clear advantage of our method ({\bf Ordinal-T}) over the matricization ({\bf Ordinal-M}) in both complete and non-complete observations. When the observation fraction is small, e.g., ${|\Omega|/ d^K}=0.4$, the tensor-based completion shows $\sim$ 30\% reduction in error compared to the matricization.

\begin{figure}[ht]
\includegraphics[width=\textwidth]{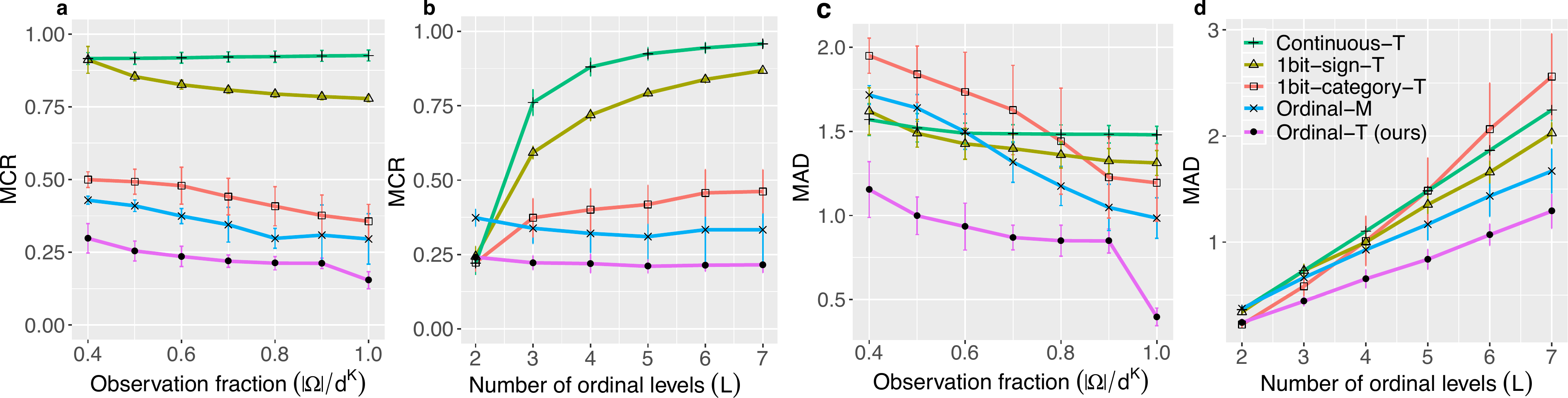}
\caption{Performance comparison for predicting most likely labels. (a, c) Prediction errors versus sample complexity $\rho=|\Omega|/d^K$ when $L=5$. (b, d) Prediction errors versus the number of ordinal levels $L$ when $\rho=0.8$. }\label{fig:compare}
\end{figure}

\begin{figure}[ht]
\includegraphics[width=1\textwidth]{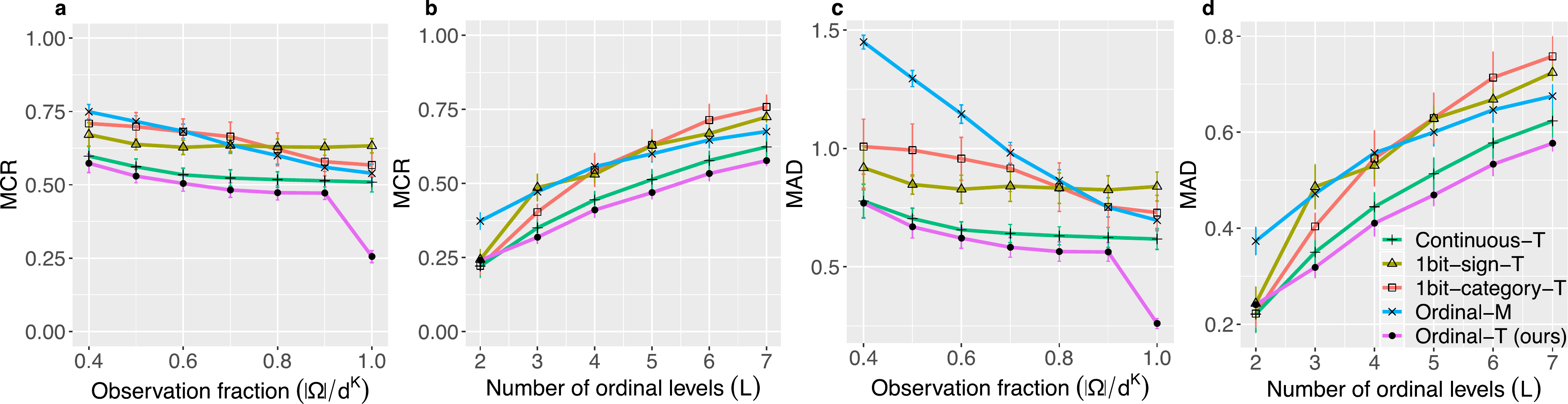}
\caption{Performance comparison for predicting median labels. (a, c) Prediction error versus sample complexity $\rho=|\Omega|/d^K$ when $L=5$. (b, d) Prediction error versus the number of ordinal levels $L$, when $\rho=0.8.$}\label{fig:compare2}
\end{figure}

We also compare the methods by their performance in predicting the median labels, $y_\omega^{\text{median}}=\min\{\ell\colon\mathbb{P}(y_\omega=\ell)\geq 0.5\}$. Under the latent variable model~\eqref{eq:latent} and Assumption~\ref{ass:link}, the median label is the quantized $\theta_\omega$ without noise; i.e.,\ $y_\omega^{\text{median}}=\sum_\ell \mathds{1}\{\theta_\omega\in(b_{\ell-1},b_\ell]\}$.
 We utilize the same simulation setting as in the earlier experiment. Figure~\ref{fig:compare2} shows that our method outperforms the others in both MCR and MAD. The improved accuracy comes from the incorporation of multilinear low-rank structure, multi-level observations, and the ordinal structure. Interestingly, the median estimator tends to yield smaller MAD than the mode estimator, $\text{MAD}(\tY^{\text{median}}, \hat \tY^{\text{median}})\leq \text{MAD}(\tY^{\text{mode}}, \hat \tY^{\text{mode}})$  (Figures~\ref{fig:compare}a-b vs.\ Figures~\ref{fig:compare2}a-b), for the three multilevel methods ({\bf 1bit-sign-T}, {\bf Ordinal-M}, and {\bf Ordinal-T}). The mode estimator, on the other hand, tends to yield smaller MCR than the median estimator, $\text{MCR}(\tY^{\text{mode}}, \hat \tY^{\text{mode}})\leq\text{MCR}(\tY^{\text{median}}, \hat \tY^{\text{median}})$ (Figures~\ref{fig:compare}c-d vs.\ Figures~\ref{fig:compare2}c-d). This tendency is from the property that the median estimator $\hat y^{(\text{median})}_\omega$ minimizes $R_1(z)=\mathbb{E}_{y_\omega}|y_\omega-z|$, whereas the mode estimator $\hat y^{(\text{mode})}_\omega $ minimizes $R_2(z)=\mathbb{E}_{y_\omega}\mathds{1}\{y_\omega=z\}$. Here the expectation is taken over the testing data $y_\omega$, conditional on the training data and thus parameters $\hat \Theta$ and~$\hat \mb$.

\section{Data Applications}\label{sec:dataapplication}
We apply our ordinal tensor method to two real-world datasets. In the first application, we use our model to analyze an ordinal tensor consisting of structural connectivities among 68 brain regions for 136 individuals from Human Connectome Project (HCP)~\citep{van2013wu}. In the second application, we perform tensor completion to an ordinal dataset with missing values. The data tensor records the ratings of 139 songs on a scale of 1 to 5 from 42 users on 26 contexts~\citep{baltrunas2011incarmusic}.

\subsection{Human Connectome Project (HCP)}
Each entry in the HCP dataset takes value on a nominal scale, \{{\it high, moderate, low}\}, indicating the strength level of fiber connection. We convert the dataset to a 3-level ordinal tensor $\tY\in[3]^{68\times 68\times 136}$ and apply the ordinal tensor method with a logistic link function. The BIC suggests $\mr = (23,23,8)$ with $\logl(\hat{\Theta},\hat{\mb}) = -216,646$. Based on the estimated Tucker factors $\{\hat \mM_k\}$, we perform a clustering analysis via K-mean on the brain nodes (see detailed procedure in the Supplement). We find that the clustering successfully captures the spatial separation between brain regions (Table~\ref{table:clustering}).
In particular, cluster I represents the connection between the left and right hemispheres, whereas clusters II-III represent the connection within each of the half brains (Figure~\ref{figure:brain image}). Other smaller clusters represent local regions driving by similar nodes (Table~\ref{table:clustering}). For example, the cluster IV/VII consists of nodes in the supramarginal gyrus region in the left/right hemisphere. This region is known to be involved in visual word recognition and reading~\citep{stoeckel2009supramarginal}. The identified similarities among nodes without external annotations illustrate the applicability of our method to clustering analysis.

\begin{figure}[ht]
\centering
\includegraphics[width = .8\textwidth]{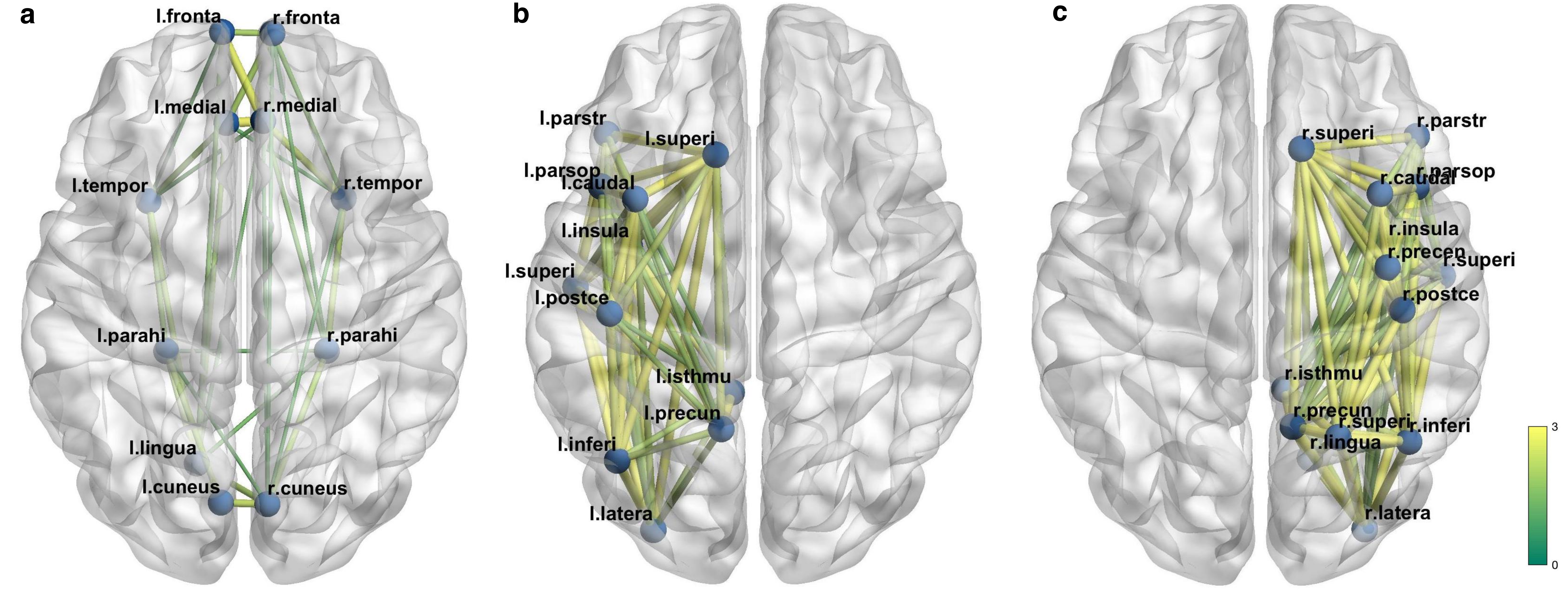}
\caption{Top three clusters in the HCP analysis. (a) Cluster I reflects the connections between two brain hemispheres. (b)-(c) Cluster II/III consists of nodes within left/right hemisphere only. Node names are shown in abbreviation. Edges are colored based on estimated connection averaged across individuals. }  \label{figure:brain image}
\end{figure}

\begin{table}[ht]
\center
\begin{sc}
\resizebox{\columnwidth}{!}{
\begin{tabular}{|c|c|c|c|}
\hline
Cluster     & \multicolumn{3}{c|}{I}\\ \hline
Brain nodes & \multicolumn{3}{l|}{\begin{tabular}[c]{@{}l@{}} l.frontalpole, l.temporalpole, l.medialorbitofrontal,  l.cuneus, l.parahippocampal, l.lingual,            \\
r.frontalpole, r.temporalpole,  r.medialorbitofrontal,  r.cuneus,  r.parahippocampal \\
\end{tabular}} \\
\hline
Cluster & \multicolumn{3}{c|}{II}\\ \hline
Brain nodes & \multicolumn{3}{l|}{\begin{tabular}[c]{@{}l@{}}
l.caudalmiddlefrontal,  l.inferiorparietal,  l.insula,  l.isthmuscingulate, l.lateraloccipital(2), \\
l.parsopercularis, l.parstriangularis, l.postcentral, l.precuneus, l.superiorfrontal, l.superiortemporal(3)
\end{tabular}} \\ 
\hline
Cluster & \multicolumn{3}{c|}{III}\\ \hline
Brain nodes & \multicolumn{3}{l|}{\begin{tabular}[c]{@{}l@{}}
r.caudalmiddlefrontal, r.inferiorparietal, r.insula, r.isthmuscingulate, r.lateraloccipital(2), r.lingual, \\
r.parsopercularis, r.parstriangularis, r.postcentral, r.precentral, r.precuneus, r.superiorfrontal(3), \\
r.superiorparietal,  r.superiortemporal(3)
\end{tabular}} \\ 
\hline
Cluster      & IV    & V  & VI\\
\hline
Brain nodes & \begin{tabular}[c]{@{}c@{}}l.supramarginal(4)\end{tabular}&
\begin{tabular}[c]{@{}c@{}}l.inferiortemporal(3)\end{tabular}&
\begin{tabular}[c]{@{}c@{}}l.middletemporal(3)\end{tabular}\\
\hline
Cluster & VII  & VIII & VIIII\\
\hline
Brain nodes &
\begin{tabular}[c]{@{}c@{}}r.supramarginal(4)\end{tabular}&
\begin{tabular}[c]{@{}c@{}}r.inferiortemporal(3)\end{tabular}&
\begin{tabular}[c]{@{}c@{}}r.middletemporal(3)\end{tabular}\\
\hline
Cluster &X & \multicolumn{2}{c|}{XI}  \\
\hline
Brain nodes &
\begin{tabular}[c]{@{}c@{}}l.superiorfrontal(2)\end{tabular}&
 \multicolumn{2}{c|}{l.precentral,  l.superiorparietal}\\
 \hline
\end{tabular}
}
\end{sc}
\caption{Node clusters in the HCP analysis. The first alphabet in the node name indicates the left (L) or right (R) hemisphere. The number in the parentheses indicates the node count in each cluster. }  \label{table:clustering}
\end{table}

We compare the goodness-of-fit of various tensor methods on the HCP data. Table~\ref{table:CV} summarizes the prediction error via 5-fold stratified cross-validation averaged over 10 runs. Our method outperforms the others, especially in MAD.
\subsection{InCarMusic recommendation system}

We apply ordinal tensor completion to a recommendation system {\it InCarMusic}. {\it InCarMusic} is a mobile application that offers music recommendation to passengers of cars based on contexts~\citep{baltrunas2011incarmusic}. We conduct tensor completion on the 42-by-139-by-26 tensor with 2,844 observed entries only. Table~\ref{table:CV} shows the averaged prediction error via 5-fold cross validation. The high missing rate makes the accurate classification challenging. Nevertheless, our method achieves the best performance among the three.

\begin{table}[ht]
\resizebox{\textwidth}{!}{
\begin{tabular}{c|c|c|c|c|c}
\Xhline{2\arrayrulewidth}
\multicolumn{3}{c|}{Human Connectome Project (HCP) dataset}&\multicolumn{3}{c}{InCarMusic dataset}\\
\Xhline{2\arrayrulewidth}
Method                    & MAD & MCR&Method                    & MAD & MCR\\
\hline
Ordinal-T (ours) & 0.1607 (0.0005)& 0.1606 (0.0005)&Ordinal-T (ours) & 1.37 (0.039) &  0.59 (0.009)\\
 Continuous-T&0.2530 (0.0002)&0.1599 (0.0002)&Continuous-T & 2.39 (0.152)&0.94 (0.027)\\
 1bit-sign-T&0.3566 (0.0010)&0.1563 (0.0010)&1bit-sign-T&1.39 (0.003)& 0.81 (0.005)\\
\hline
\end{tabular}
}
\caption{Comparison of prediction error in the HPC and {\it InCarMusic} analyses. Standard errors are reported in parentheses.}\label{table:CV}
\end{table}

\section{Proofs}
Here, we provide additional theoretical results and proofs presented in Sections~\ref{sec:theory}.

\subsection{Extension of Theorem~\ref{thm:rate} to unknown cut-off points}\label{sec:extention}
We now extend Theorem~\ref{thm:rate} to the case of unknown cut-off points $\mb$. Assume that the true parameters $(\trueT, \trueb)\in \tP\times \tB$, where the feasible sets are defined as
\begin{align}
\tP&=\{\Theta\in\mathbb{R}^{d_1\times \cdots\times d_K}\colon \text{rank}(\tP)\leq \mr,\ \langle \Theta, \tJ \rangle=0,\ \mnormSize{}{\Theta}\leq \alpha\},\\
 \tB&=\{\mb\in\mathbb{R}^{L-1}\colon \mnormSize{}{\mb}\leq \beta,\ \min_{\ell}(b_\ell-b_{\ell-1}) \geq \Delta\},
\end{align}
with positive constants $\alpha,\beta,\Delta>0$ and a given rank $\mr\in\mathbb{N}^K_{+}$. Here, $\tJ=\entry{1}\in\mathbb{R}^{d_1\times \cdots \times d_K}$ denotes a tensor of all ones. The constraint $\langle \Theta, \tJ \rangle=0$ is imposed to ensure the identifiability of $\Theta$ and $\mb$. We propose the constrained M-estimator
\begin{equation}\label{eq:joint}
(\hat \Theta,\hat \mb)=\argmax_{(\Theta,\mb)\in\tP\times \tB}\flogl(\Theta,\mb).
\end{equation}
The estimation accuracy is assessed using the mean squared error (MSE):
\[
\text{MSE}\left( \hat \Theta, \trueT \right)={1\over \prod_k d_k}\FnormSize{}{\hat \Theta-\trueT}^2,\quad \text{MSE}\left( \hat \mb, \trueb \right)={1\over L-1}\FnormSize{}{\hat \mb-\trueb}^2.
\]

To facilitate the examination of MSE, we define an order-$(K+1)$ tensor, $\tZ=\entry{z_{\omega,\ell}}\in\mathbb{R}^{d_1\times \cdots \times d_K\times (L-1)}$, by stacking the parameters $\Theta=\entry{\theta_\omega}$ and $\mb=\entry{b_\ell}$ together. Specifically, let $z_{\omega,\ell}=-\theta_\omega+b_\ell$ for all $\omega\in[d_1]\times \cdots \times [d_K]$ and $\ell\in[L-1]$; that is,
\[
\tZ=-\Theta\otimes \mathbf{1}+\tJ\otimes \mb,
\]
where $\mathbf{1}$ denotes a length-$(L-1)$ vector of all ones. Under the identifiability constraint $\langle \Theta, \tJ \rangle=0$, there is an one-to-one mapping between $\tZ$ and $(\Theta, \mb)$, with $\text{rank}(\tZ)\leq (\text{rank}(\Theta)+1,\ 2)^T$. 
Furthermore, 
\begin{equation}\label{eq:decomp}
\FnormSize{}{\hat \tZ-\trueZ}^2=\FnormSize{}{\hat \Theta-\trueT}^2 (L-1)+\FnormSize{}{\hat \mb-\trueb}^2 \left(\prod_kd_k\right),
\end{equation}
where $\trueZ=-\trueT\otimes \mathbf{1}+\tJ\otimes \trueb$ and $\hat \tZ=-\hat \Theta\otimes \mathbf{1}+\tJ\otimes \hat \mb$. 

We make the following assumptions about the link function.
\begin{assumption}\label{ass:joint}
The link function $f\colon \mathbb{R}\mapsto [0,1]$ satisfies the following properties:
\begin{enumerate}[label=\textup{(\roman*)}]
\item $f(z)$ is twice-differentiable and strictly increasing in $z$.
\item $\dot{f}(z)$ is strictly log-concave and symmetric with respect to $z=0$.
\end{enumerate}
\end{assumption}

We define the following constants that will be used in the theory:
\begin{align}\label{eq:constantZ}
C_{\alpha,\beta,\Delta}&=\max_{|z|\leq \alpha+\beta}\max_{\substack{z'\leq z-\Delta\\z''\geq z+\Delta}}\max\left\{ {\dot{f}(z) \over f(z)-f(z')},\ {\dot{f}(z) \over f(z'')-f(z)}\right\},\\
D_{\alpha,\beta,\Delta}&=\min_{|z|\leq \alpha+\beta}\min_{\substack{z'\leq z-\Delta\\z''\geq z+\Delta}}\min\left\{- {\partial\over \partial z}\left({\dot{f}(z) \over f(z)-f(z')}\right),\  {\partial\over \partial z}\left({\dot{f}(z) \over f(z'')-f(z)}\right) \right\},\\
A_{\alpha,\beta,\Delta}&=\min_{|z|\leq \alpha+\beta}\min_{z'\leq z-\Delta} \left(f(z)-f(z')\right).
\end{align}

\begin{rmk}
The condition $\Delta=\min_{\ell}(b_{\ell}-b_{\ell-1})>0$ on the feasible set $\tB$ guarantees the strict positiveness of $f(z)-f(z')$ and $f(z'')-f(z)$. Therefore, the denominators in the above quantities $C_{\alpha,\beta,\Delta},\ D_{\alpha,\beta,\Delta}$ are well-defined. Furthermore, by Theorem~\ref{thm:convexity}, $f(z)-f(z')$ is strictly log-concave in $(z,z')$ for $z\leq z'-\Delta,\ z,z'\in[-\alpha-\beta,\ \alpha+\beta]$. Based on Assumption~\ref{ass:joint} and closeness of the feasible set, we have $C_{\alpha,\beta,\Delta}>0$, $D_{\alpha,\beta,\Delta}>0$, $A_{\alpha,\beta,\Delta}>0$.
\end{rmk}
\begin{rmk} In particular, for logistic link $f(x)= \frac{1}{1+e^{-x}}$, we have
\begin{align}
C_{\alpha,\beta,\Delta}&=\max_{|z|\leq \alpha+\beta}\max_{\substack{z'\leq z-\Delta\\z''\geq z+\Delta}}\max\left\{ \frac{1}{e^\Delta-1}\left(\frac{1+e^{-z'}}{1+e^{-z}}\right),\ \frac{1}{1-e^{-\Delta}}\left(\frac{1+e^{-z''}}{1+e^{-z}}\right)\right\}>0,\\
D_{\alpha,\beta,\Delta}&=\min_{|z|\leq \alpha+\beta}\frac{e^{-z}}{(1+e^{-z})^2}>0.\end{align}
\end{rmk}
\begin{thm}[Statistical convergence with unknown $\mb$]\label{thm:ratejoint}
Consider an ordinal tensor $\tY\in[L]^{d_1\times \cdots \times d_K}$ generated from model~\eqref{eq:model} with the link function $f$ and parameters $(\trueT, \trueb)\in\tP\times \tB$. Suppose the link function $f$ satisfies Assumption~\ref{ass:joint}. Define $r_{\max}=\max_k r_k+1$, and assume $r_{\max}=\tO(1)$. 

Then with very high probability, the estimator in~\eqref{eq:joint} satisfies
\begin{equation}\label{eq:scoreZ}
\FnormSize{}{\hat \tZ-\trueZ}^2 \leq {c_1r^{K}_{\max}C^2_{\alpha,\beta,\Delta}\over A^2_{\alpha,\beta,\Delta}D^2_{\alpha,\beta,\Delta}}\left(L-1+\sum_k d_k\right),
\end{equation}
In particular,
\begin{equation}\label{eq:jointTheta}
\mathrm{MSE}\left(\hat \Theta, \trueT \right)\leq\min\left\{ 4\alpha^2,\  {c_1 r_{\max}^{K}  C_{\alpha,\beta,\Delta}^2 \over A^2_{\alpha,\beta,\Delta}D_{\alpha,\beta,\Delta}^2}\left({ L-1+\sum_k d_k \over \prod_k d_k}\right)\right\},
\end{equation}
and
\begin{equation}\label{eq:jointb}
\mathrm{MSE}\left(\hat \mb, \trueb \right)\leq \min\left\{4\beta^2,\ {c_1 r^{K}_{\max} C^2_{\alpha,\beta,\Delta} \over A^2_{\alpha,\beta,\Delta}D_{\alpha,\beta,\Delta}^2}\left({L-1+\sum_kd_k \over \prod_k d_K}\right)\right\},
\end{equation}
where $c_1, C_{\alpha,\beta,\Delta}, D_{\alpha,\beta,\Delta}$ are positive constants independent of the tensor dimension, rank, and number of ordinal levels. 
\end{thm}
\begin{proof} 
The log-likelihood associated with the observed entries in terms of $\tZ$ is
\begin{equation}\label{eq:loglz}
\flogl(\tZ) = \sum_{\omega\in\Omega}\sum_{\ell\in[L]}\mathds{1}\{y_\omega=\ell\}\log\left[f(z_{\omega,\ell})-f(z_{\omega,\ell-1})\right].
\end{equation}

Let $\plogZ=\entry{{\partial \tL_\tY\over \partial z_{\omega,\ell}}}\in\mathbb{R}^{d_1\times \cdots \times d_K\times[L-1]}$ denote the score function, and $\mH=\pplogZ$ the Hession matrix. 
Based on the definition of $\hat\tZ$, we have the following inequality:
\begin{equation}\label{eq:ineql}
\flogl(\hat\tZ) \geq \flogl(\trueZ).
\end{equation}

Following the similar argument in Theorem \ref{thm:rate} and the inequality \eqref{eq:ineql}, we obtain that
\begin{equation}\label{eq:boundZ}
\FnormSize{}{\hat \tZ-\trueZ}^2 \leq c_1r^{K}_{\max}{\snormSize{}{\plogZ(\trueZ)}^2\over \lambda^2_{1} \left(\mH(\check\tZ)\right)},
\end{equation}
where $\plogZ(\trueZ)$ is the score evaluated at $\trueZ$, $\mH(\check \tZ)$ is the Hession evaluated at $\check\tZ$, for some $\check\tZ$ between $\hat \tZ$ and $\trueZ$,  and $\lambda_{1}(\cdot)$ is the largest matrix eigenvalue. 

We bound the score and the Hessian to obtain \eqref{eq:scoreZ}.
\begin{enumerate}[itemsep=0pt,topsep=0pt,leftmargin=*,partopsep=0pt]
\item (Score.) The $(\omega,\ell)$-th entry in $\plogZ$ is 
\[
{\partial \tL_\tY\over \partial z_{\omega,\ell}}=\mathds{1}\{y_\omega=\ell\}{\dot{f}(z)\over f(z)-f(z')}\Bigg|_{(z,\ z')=(z_{\omega,\ell},\ z_{\omega,\ell-1})} - \mathds{1}\{y_\omega=\ell+1\}{\dot{f}(z)\over f(z'')-f(z)}\Bigg|_{(z'',\ z)=(z_{\omega,\ell+1},\ z_{\omega,\ell})},
\]
which is upper bounded in magnitude by $C_{\alpha,\beta,\Delta}>0$ with zero mean. By Lemma \ref{lem:noisytensor}, with probability at least $1-\exp\left(-c_2'
\left(\sum_k d_k+ L-1\right)\right)$, we have
\begin{equation}\label{eq:scorebound}
\snormSize{}{\plogZ(\trueZ)}\leq c_2C_{\alpha,\beta,\Delta}\sqrt{L-1+\sum_k d_k},
\end{equation}

where $c_2,c_2'$ are two positive constants that depend only on $K$.
\item (Hession.) The entries in the Hession matrix are
\begin{align}
\text{Diagonal: }&{\partial^2 \tL_\tY\over \partial z^2_{\omega,\ell}}=\mathds{1}\{y_\omega=\ell\}{\ddot{f}(z)\left(f(z)-f(z')\right)-\dot{f}^2(z)\over \left(f(z)-f(z')\right)^2}\Bigg|_{(z,\ z')=(z_{\omega,\ell},\ z_{\omega,\ell-1})}-\\
&\hspace{.65in}\mathds{1}\{y_\omega=\ell+1\}{\ddot{f}(z)\left(f(z'')-f(z)\right)+\dot{f}^2(z)\over \left(f(z'')-f(z)\right)^2}\Bigg|_{(z'',\ z)=(z_{\omega,\ell+1},\ z_{\omega,\ell})},\\
\text{Off-diagonal: }&
{\partial^2 \tL_\tY\over \partial z_{\omega,\ell}z_{\omega,\ell+1}}=\mathds{1}\{y_\omega=\ell+1\}{\dot{f}(z_{\omega,\ell})\dot{f}(z_{\omega,\ell+1})\over \left(f(z_{\omega,\ell+1})-f(z_{\omega,\ell})\right)^2}\quad \text{and}\quad {\partial^2 \tL_\tY\over \partial z_{\omega,\ell}z_{\omega',\ell'}}=0 \text{ otherwise}. 
\end{align}

Based on Assumption~\ref{ass:joint}, the Hession matrix $\mH$ has the following three properties:
\begin{enumerate}[label=\textup{(\roman*)}]
\item The Hession matrix is a block matrix, $\mH=\text{diag}\{\mH_\omega\colon \omega\in[d_1]\times \cdots \times [d_K]\}$, and each block $\mH_{\omega}\in\mathbb{R}^{(L-1)\times (L-1)}$ is a tridiagonal matrix. 
\item The off-diagonal entries are either zero or strictly positive.
\item The diagonal entries are either zero or strictly negative. Furthermore, 
\begin{align}
&\mH_\omega(\ell,\ell)+ \mH_\omega(\ell,\ell-1)+\mH_\omega(\ell,\ell+1)
\\
=&\ {\partial^2 \tL_\tY\over \partial z^2_{\omega,\ell}}+{\partial^2 \tL_\tY\over \partial z_{\omega,\ell}z_{\omega,\ell+1}}+{\partial^2 \tL_\tY\over \partial z_{\omega,\ell-1}z_{\omega,\ell}}\\
=&\ \mathds{1}\{y_\omega=\ell\}{\partial\over \partial z}\left({\dot{f}(z) \over f(z)-f(z')}\right)\Bigg|_{(z,\ z')=(z_{\omega,\ell},\ z_{\omega,\ell-1})} \\& -\mathds{1}\{y_\omega=\ell+1\}{\partial\over \partial z}\left({\dot{f}(z) \over f(z)-f(z')}\right)\Bigg|_{(z'',\ z)=(z_{\omega,\ell+1},\ z_{\omega,\ell})}\\
\leq &\ -D_{\alpha,\beta,\Delta}\mathds{1}\{y_\omega = \ell \text{ or } \ell+1\}.
\end{align}
\end{enumerate}
We will show that, with very high probability over $\tY$, $\mH$ is negative definite in that
\begin{equation}\label{eq:HessionZ}
\lambda_{1}(\mH)=\max_{\mz\neq 0}{\mz^T\mH\mz\over\FnormSize{}{\mz}^2} \leq - cA_{\alpha,\beta,\Delta}D_{\alpha,\beta,\Delta},
\end{equation}
where $A_{\alpha,\beta,\Delta},\ D_{\alpha,\beta,\Delta}>0$ are constants defined in~\eqref{eq:constantZ}, and $c>0$ is a constant. 

Let $\mz_\omega=(z_{\omega,1},\ldots,z_{\omega,L-1})^T\in\mathbb{R}^{L-1}$ and $\mz=(\mz_{1,\ldots,1,1},\ldots,\mz_{d_1,\ldots,d_K,L-1})^T\in\mathbb{R}^{(L-1)\prod_k d_k}$. It follows from property (i) that
\[
\mz^T\mH\mz =\sum_{\omega}\mz_\omega^T \mH_{\omega} \mz_\omega.
\]
Furthermore, from properties (ii) and (iii) we have
\begin{align}
\mz_\omega^T \mH_{\omega} \mz_\omega &=
\sum_{\ell\in[L-1]}\mH_\omega(\ell,\ell)z_{\omega,\ell}^2+\sum_{\ell\in[L-1]/\{1\}}2\mH_{\omega}(\ell,\ell-1)z_{\omega,\ell}z_{\omega,\ell-1}\\
&\leq \sum_{\ell\in[L-1]}\mH(\ell,\ell)z^2_{\omega,\ell}+\sum_{\ell\in[L-1]/\{1\}}\mH(\ell,\ell-1)\left[z_{\omega,\ell}^2+z_{\omega,\ell-1}^2\right]\\
&= \left(\mH(1,1)+\mH(1,2)\right)z_{\omega,1}^2+ \left(\mH(L-1,L-1)+\mH(L-1,L-2)\right)z_{\omega,L-1}^2\\&\quad\quad+\sum_{\ell\in[L-2]/\{1\}}\left(\mH(\ell,\ell)+\mH(\ell,\ell-1)+\mH(\ell,\ell+1)\right)z_{\omega,\ell}^2\\
& \leq -D_{\alpha,\beta,\Delta} \sum_{\ell}z^2_{\omega,\ell}\mathds{1}\{y_\omega=\ell \text{ or }\ell+1\}.
\end{align}
Therefore,
\begin{equation}\label{eq:sum}
\mz^T\mH\mz =\sum_{\omega}\mz_\omega^T \mH_{\omega} \mz_\omega \leq -D_{\alpha,\beta,\Delta}\sum_{\omega}\sum_{\ell} z^2_{\omega,\ell}\mathds{1}\{y_\omega=\ell \text{ or }\ell+1\}.
\end{equation}

Define the subspace:
\[
\tS=\{\Vec(\tZ): \tZ=-\Theta\otimes \mathbf{1}+\tJ\otimes \mb,\ (\Theta,\mb)\in(\tP,\tB)\}.
\]
It suffices to prove the negative definiteness of Hession when restricted in the subspace $\tS$. Specifically, for any vector $\mz=\entry{z_{\omega,\ell}}\in\tS$,
\begin{align}
\sum_{\omega,\ell} z^2_{\omega,\ell}\mathds{1}\{y_\omega=\ell \text{ or }\ell+1\}&=\sum_{\omega,\ell}(-\theta_\omega+b_\ell)^2\mathds{1}\{y_\omega=\ell \text{ or }\ell+1\}\\
&=\sum_{\omega,\ell}(\theta^2_\omega-2\theta_\omega b_\ell + b^2_\ell)\mathds{1}\{y_\omega=\ell \text{ or }\ell+1\}\\
&\geq \sum_{\omega}\theta^2_\omega -2\sum_{\omega,\ell}\theta_\omega b_\ell+\sum_{\ell}b_\ell^2\left(n_{\ell}+n_{\ell+1}\right)\\
&\geq \sum_{\omega}\theta^2_\omega+\min_\ell\left(n_{\ell}+n_{\ell+1}\right)\sum_{\ell}b_\ell^2.
\end{align}

On the other hand,
\[
\FnormSize{}{\mz}^2=\sum_{\omega,\ell}z^2_{\omega,\ell}=\sum_{\omega,\ell}(-\theta_\omega+b_\ell)^2=L_{\text{total}}\sum_{\omega}\theta^2_\omega+d_{\text{total}}\sum_\ell b^2_\ell,
\]
where $L_{\text{total}}:=(L-1)$ and $d_{\text{total}}:=\prod_k d_k$.

Therefore, we have
\begin{align}\label{eq:negativeH}
\max_{\mz\in\tS,\mz\neq \mathbf{0}}{\sum_{\omega,\ell} z^2_{\omega,\ell}\mathds{1}\{\{y_\omega=\ell \text{ or }\ell+1\}\} \over \FnormSize{}{\mz}^2}&\geq {\sum_{\omega}\theta^2_\omega+\min_\ell\left(n_{\ell}+n_{\ell+1}\right)\sum_{\ell}b_\ell^2\over L_{\text{total}} \sum_{\omega}\theta^2_\omega+d_{\text{total}}\sum_\ell b^2_\ell}
\geq{\min_{\ell}(n_\ell+n_{\ell+1})\over (1+{\alpha^2\over c\Delta^2})d_{\text{total}}}\notag\\
&\geq {2A_{\alpha,\beta,\Delta}\over 1+{\alpha^2\over c\Delta^2}}\quad \text{in high probability as $d_{\min}\to \infty$}. 
\end{align}
The second inequality in \eqref{eq:negativeH} is from the conditions that 
\[
\sum_{\omega}\theta_\omega^2\in [0,\ \alpha^2 d_{\text{total}}]\quad \text{and}\quad \sum_\ell b_\ell^2 \in[c L_{\text{total}}\Delta^2,\  L_{\text{total}}\beta^2],
\]
for some universal constant $c>0$. 
The last inequality in \eqref{eq:negativeH} follows by applying the law of large numbers and the uniform bound $\min_{z_{\omega,\ell}}\mathbb{P}(y_{\omega}=\ell \text{ or }\ell+1|z_{\omega,\ell})\geq 2A_{\alpha,\beta,\Delta}$ to the empirical ratio:
\[
{\min_{\ell}(n_\ell+n_{\ell+1})\over d_{\text{total}}}\overset{p}{\to}\min_{\ell} \mathbb{P}(y_{\omega}=\ell \text{ or }\ell+1|z_{\omega,\ell}) \geq 2A_{\alpha,\beta,\Delta}, \quad \text{ in high probability as $d_{\min}\to \infty$}.
\] 

By~\eqref{eq:sum} and~\eqref{eq:negativeH}, we have
\begin{equation}\label{eq:boundH}
\mz^T\mH\mz \leq- c'A_{\alpha,\beta,\Delta}D_{\alpha,\beta,\Delta}\FnormSize{}{\mz}^2,
\end{equation}
for some constant $c'>0$, therefore~\eqref{eq:HessionZ} is proved.
\end{enumerate}
Finally, plugging~\eqref{eq:scorebound} and~\eqref{eq:HessionZ} into~\eqref{eq:boundZ} yields
\[
\FnormSize{}{\hat \tZ-\trueZ}^2 \leq {c_1r^{K}_{\max}C^2_{\alpha,\beta,\Delta}\over A^2_{\alpha,\beta,\Delta}D^2_{\alpha,\beta,\Delta}}\left(L-1+\sum_k d_k\right).
\]
The MSEs for $\hat \Theta$ and $\hat \mb$ readily follow from~\eqref{eq:decomp}. 
\end{proof}

\subsection{Proof of Theorem~\ref{thm:rate}}\label{sec:proofMSE}
\begin{proof}[Proof of Theorem~\ref{thm:rate}]
We suppress the subscript $\Omega$ in the proof, because the tensor denoising assumes complete observation $\Omega=[d_1]\times \cdots \times [d_K]$. It follows from the expression of $\flogl(\Theta)$ that
\begin{align}\label{eq:property}
{\partial \flogl\over \partial \theta_\omega}&=\sum_{\ell\in[L]}\mathds{1}\{y_{\omega}=\ell\}
{\dot{g}_\ell(\theta_\omega)\over g_\ell(\theta_\omega)},\notag\\
{\partial^2 \flogl\over \partial \theta_\omega^2}&=\sum_{\ell\in[L]}\mathds{1}\{y_\omega=\ell\}{\ddot{g}_\ell(\theta_\omega)g_\ell(\theta_\omega)-\dot{g}^2_\ell(\theta_\omega)\over g^2_\ell(\theta_\omega)}\ \text{and}\quad
{\partial^2 \flogl\over \partial \theta_\omega \theta_\omega'}=0\ \text{if}\ \omega\neq \omega',
\end{align}
for all $\omega\in[d_1]\times \cdots \times [d_K]$.
Define $d_{\text{total}}=\prod_k d_k$. Let $\fplogl\in\mathbb{R}^{d_1\times\cdots\times d_K}$ denote the tensor of gradient with respect to $\Theta\in\mathbb{R}^{d_1\times \cdots\times d_K}$, and $\fpplogl$ the corresponding Hession matrix of size $d_\text{total}$-by-$d_{\text{total}}$. Here, $\Vec(\cdot)$ denotes the operation that turns a tensor into a vector. By~\eqref{eq:property}, $\fpplogl$ is a diagonal matrix. Recall that
\begin{equation}\label{eq:bound}
U_\alpha=\max_{\ell\in[L],|\theta|\leq \alpha}{|\dot{g}_\ell(\theta)|\over g_\ell(\theta)}>0 \quad \text{and}\quad
L_\alpha=\min_{\ell\in[L],|\theta|\leq \alpha} {\dot{g}^2_\ell(\theta)-\ddot{g}_\ell(\theta)g_\ell(\theta)\over g^2_\ell(\theta)}>0.
\end{equation}
Therefore, the entries in $\fplogl$ are upper bounded in magnitude by $U_\alpha>0$, and all diagonal entries in $\fpplogl$ are upper bounded by $-L_{\alpha}<0$.

By the second-order Taylor's expansion of $\flogl(\Theta)$ around $\trueT$, we obtain
\begin{equation}\label{eq:taylor}
\flogl(\Theta)=\flogl(\trueT)+\langle\Vec(\fplogl(\trueT)),\ \Vec(\Theta-\trueT)\rangle+{1\over 2}\Vec(\Theta-\trueT)^T\fpplogl(\check\Theta)\Vec(\Theta-\trueT),
\end{equation}
where $\check\Theta=\gamma\trueT+(1-\gamma)\Theta$ for some $\gamma\in[0,1]$, and $\fpplogl(\check\Theta)$ denotes the $d_{\text{total}}$-by-$d_\text{total}$ Hession matrix evaluated at $\check\Theta$.

We first bound the linear term in~\eqref{eq:taylor}. Note that, by Lemma~\ref{lem:inq},
\begin{equation}\label{eq:linear}
|\langle\Vec(\fplogl(\trueT), \Vec(\Theta-\trueT)  \rangle|\leq \snormSize{}{\fplogl(\trueT)} \nnormSize{}{\Theta-\trueT},
\end{equation}
where $\snormSize{}{\cdot}$ denotes the tensor spectral norm and $\nnormSize{}{\cdot}$ denotes the tensor nuclear norm. Define
\[
s_\omega={\partial \tL_\tY\over \partial \theta_\omega}\Big|_{\Theta=\trueT} \;\; \textrm{ for all } \; \omega\in[d_1]\times\cdots\times [d_K].
\]
Based on~\eqref{eq:property} and the definition of $U_\alpha$, $\fplogl(\trueT)=\entry{s_{\omega}}$ is a random tensor whose entries are independently distributed satisfying
\begin{equation}\label{eq:norm}
\mathbb{E}(s_\omega)=0,\quad |s_\omega|\leq U_\alpha, \quad \text{for all }\omega\in[d_1]\times \cdots \times [d_K].
\end{equation}
By lemma~\ref{lem:noisytensor}, with probability at least $1-\exp(-C_1 \sum_kd_k)$, we have
\begin{equation}\label{eq:normrandom}
\snormSize{}{\fplogl(\trueT)} \leq C_2 U_\alpha\sqrt{\sum_k d_k},
\end{equation}
where $C_1, C_2$ are two positive constants that depend only on $K$. Furthermore, note that $\text{rank}(\Theta)\leq \mr$, $\text{rank}(\trueT)\leq \mr$, so $\text{rank}(\Theta-\trueT)\leq 2\mr$. By lemma~\ref{lem:nuclear}, $\nnormSize{}{\Theta-\trueT}\leq (2r_{\max})^{K-1\over 2}\FnormSize{}{\Theta-\trueT}$. Combining~\eqref{eq:linear}, \eqref{eq:norm} and \eqref{eq:normrandom}, we have that, with probability at least $1-\exp(-C_1 \sum_kd_k)$,
\begin{equation}\label{eq:linearconclusion}
|\langle \Vec(\fplogl(\trueT)), \Vec(\Theta-\trueT)  \rangle | \leq C_2 U_\alpha  \sqrt{r_{\max}^{K-1} \sum_k d_k}  \FnormSize{}{\Theta-\trueT}.
\end{equation}

We next bound the quadratic term in \eqref{eq:taylor}. Note that
\begin{align}\label{eq:quadratic}
 \Vec(\Theta-\trueT)^T \fpplogl(\check{\Theta})\Vec(\Theta-\trueT)&=\sum_\omega \left( {\partial^2\tL_{\tY}\over \partial \theta^2_\omega} \Big|_{\Theta=\check\Theta} \right)(\theta_\omega-\theta_{{\text{true}},\omega})^2 \nonumber \\
&\leq - L_\alpha\sum_{\omega}(\Theta_{\omega}-\Theta_{\text{true},\omega})^2 \nonumber \\
&=-L_\alpha\FnormSize{}{\Theta-\trueT}^2,
\end{align}
where the second line comes from the fact that  $\mnormSize{}{\check\Theta}\leq \alpha$ and the definition of $L_\alpha$.

Combining~\eqref{eq:taylor}, \eqref{eq:linearconclusion} and~\eqref{eq:quadratic}, we have that, for all $\Theta\in\tP$, with probability at least $1-\exp(-C_1 \sum_kd_k)$,
\[
\tL_\tY(\Theta)\leq \tL_{\tY}(\trueT)+C_2U_\alpha  \left(r_{\max}^{K-1}\sum_k d_k\right)^{1/2}  \FnormSize{}{\Theta-\trueT}-{L_\alpha\over 2}\FnormSize{}{\Theta-\trueT}^2.
\]
In particular, the above inequality also holds for $\hat \Theta\in\tP$. Therefore,
\[
\tL_\tY(\hat \Theta)\leq \tL_{\tY}(\trueT)+C_2U_\alpha \left(r_{\max}^{K-1}\sum_k d_k\right)^{1/2}  \FnormSize{}{\hat \Theta-\trueT}-{L_\alpha\over 2} \FnormSize{}{\hat \Theta-\trueT}^2.
\]
Since $\hat \Theta=\arg\max_{\Theta\in\tP}\tL_\tY(\Theta)$, $\tL_\tY(\hat \Theta)-\tL_{\tY}(\trueT)\geq 0$, which gives
\[
C_2U_\alpha \left(r_{\max}^{K-1}\sum_k d_k\right)^{1/2}  \FnormSize{}{\hat \Theta-\trueT}-{L_\alpha\over 2}\FnormSize{}{\hat \Theta-\trueT}^2\geq 0.
\]
Henceforth,
\[
{1\over \sqrt{\prod_k d_k}} \FnormSize{}{\hat \Theta-\trueT}\leq {2C_2U_\alpha \sqrt{r_{\max}^{K-1}\sum_k d_k}\over L_\alpha \sqrt{\prod_k d_k}}={2C_2U_\alpha r_{\max}^{(K-1)/2}\over L_\alpha} \sqrt{ \sum_k d_k \over \prod_k d_k}.
\]
This completes the proof.
\end{proof}

\begin{proof}[Proof of Corollary~\ref{cor:prediction}]
The result follows immediately from Theorem~\ref{thm:rate} and Lemma~\ref{lem:KL}.
\end{proof}

\subsection{Proof of Theorem~\ref{thm:minimax}}
\begin{proof}[Proof of Theorem~\ref{thm:minimax}]

Let $d_{\text{total}}=\prod_{k\in[K]}d_k$, and $\gamma\in[0,1]$ be a constant to be specified later.  Our strategy is to construct a finite set of tensors $\tX=\{\Theta_i\colon i=1,\ldots \}\subset \tP$ satisfying the properties of (i)-(iv) in Lemma~\ref{lem:construction}. By Lemma~\ref{lem:construction}, such a subset of tensors exist. For any tensor  $\Theta\in\tX$, let $\mathbb{P}_{\Theta}$ denote the distribution of $\tY|\Theta$, where $\tY$ is the ordinal tensor. In particular, $\mathbb{P}_{\mathbf{0}}$ is the distribution of $\tY$ induced by the zero parameter tensor $\mathbf{0}$, i.e., the distribution of $\tY$ conditional on the parameter tensor $\Theta=\mathbf{0}$. Based on the Remark for Lemma~\ref{lem:KL}, we have
\begin{equation}\label{eq:KLbound1}
\mathrm{KL}(\mathbb{P}_{\Theta}|| \mathbb{P}_{\mathbf{0}})\leq C \FnormSize{}{\Theta}^2,
\end{equation}
where $C={(4L-6) \dot{f}^2(0)\over  A_\alpha}>0$ is a constant independent of the tensor dimension and rank.
Combining the inequality~\eqref{eq:KLbound1} with property (iii) of $\tX$, we have
\begin{equation}\label{eq:KLbound}
\text{KL}(\mathbb{P}_{\Theta}||\mathbb{P}_{\mathbf{0}})\leq \gamma^2 r_{\max} d_{\max}.
\end{equation}
From~\eqref{eq:KLbound} and the property (i), we deduce that the condition
\begin{equation}\label{eq:totalKL}
{1\over \text{Card}(\tX)-1}\sum_{\Theta \in\tX}\text{KL}(\mathbb{P}_{\Theta}, \mathbb{P}_{\mathbf{0}})\leq \varepsilon \log\left\{\text{Card}(\tX)-1 \right\}
\end{equation}
holds for any $ \varepsilon \geq 0$ when $\gamma\in[0,1]$ is chosen to be sufficiently small depending on $\varepsilon$, e.g., $\gamma \leq \sqrt{\varepsilon\log2\over8}$. By applying Lemma~\ref{lem:Tsybakov} to~\eqref{eq:totalKL}, and in view of the property (iv), we obtain that
\begin{equation}\label{eq:final}
\inf_{\hat \Theta}\sup_{\trueT\in \tX}\mathbb{P}\left(\FnormSize{}{\hat \Theta- \trueT}\geq  {\gamma\over 8} \min\left\{ \alpha\sqrt{d_{\text{total}}}, C^{-1/2}\sqrt{ r_{\max}d_{\max}}\right\} \right)\geq {1\over 2}\left(1-2\varepsilon-\sqrt{16 \varepsilon \over r_{\max}d_{\max}\log2}\right).
\end{equation}
Note that $\textup{MSE}(\hat \Theta, \trueT)=\FnormSize{}{\hat \Theta- \trueT}^2/d_{\text{total}}$ and $\tX\subset \tP$. By taking $\varepsilon=1/10$ and $\gamma=1/11$, we conclude from~\eqref{eq:final} that
\begin{equation}\label{eq:prob}
\inf_{\hat \Theta}\sup_{\trueT\in \tP}\mathbb{P}\left(\textup{MSE}(\hat \Theta, \trueT)\geq c\min\left \{ \alpha^2,  {C^{-1}r_{\max}d_{\max}\over d_{\text{total}}}\right \}\right)\geq {1\over 2}\left({4\over 5}- \sqrt{1.6\over r_{\max}d_{\max}\log2} \right)\geq {1\over 8},
\end{equation}
where $c = {1 \over 88^2}$ and the last inequality comes from the condition for $d_{\text{max}}$.
This completes the proof.
\end{proof}

\subsection{Proof of Theorem~\ref{thm:completion}}

\begin{proof}[Proof of Theorem~\ref{thm:completion}]

For notational convenience, we use $\MFnormSize{}{\Theta}=\sum_{\omega\in\Omega}\Theta^2_\omega$ to denote the sum of squared entries over the observed set $\Omega$, for a tensor $\Theta\in\mathbb{R}^{d_1\times \cdots \times d_K}$.

Following a similar argument as in the proof of Theorem~\ref{thm:rate}, we have
\begin{equation}\label{eq:Taylor2}
\logl(\Theta)=\logl(\trueT)+\langle\Vec(\plogl),\ \Vec(\Theta-\trueT)\rangle+{1\over 2}\Vec(\Theta-\trueT)^T\pplogl(\check\Theta)\Vec(\Theta-\trueT),
\end{equation}
where
\begin{enumerate}[label=\textup{(\roman*)}]
\item $\plogl$ is a $d_1\times\cdots\times d_K$ tensor with $|\Omega|$ nonzero entries, and each entry is upper bounded by $U_\alpha>0$.
\item $\pplogl$ is a diagonal matrix of size $d_{\text{total}}$-by-$d_{\text{total}}$ with $|\Omega|$ nonzero entries, and each entry is upper bounded by $-L_{\alpha}<0$.
\end{enumerate}

Similar to~\eqref{eq:linear} and~\eqref{eq:quadratic}, we have
\begin{equation}\label{eq:linear2}
|\langle\Vec(\plogl),\ \Vec(\Theta-\trueT)\rangle|\leq C_2U_\alpha \sqrt{r_{\max}^{K-1}\sum_k d_k}\MFnormSize{}{\Theta-\trueT}
\end{equation}
and
\begin{equation}\label{eq:quadratic2}
\Vec(\Theta-\trueT)^T\fpplogl(\check\Theta)\Vec(\Theta-\trueT)\leq -L_\alpha \MFnormSize{}{\Theta-\trueT}^2.
\end{equation}

Combining~\eqref{eq:Taylor2}-\eqref{eq:quadratic2} with the fact that $\logl(\hat \Theta)\geq \logl(\trueT)$, we have
\begin{equation}\label{eq:sample}
\MFnormSize{}{\hat \Theta-\trueT}\leq {2C_2U_\alpha  r_{\max}^{(K-1)/2} \over L_\alpha} \sqrt{\sum_kd_k},
\end{equation}
with probability at least $1-\exp(-C_1 \sum_k d_k)$. Lastly, we invoke the result regarding the closeness of $\Theta$ to its sampled version $\Theta_{\Omega}$, under the entrywise bound condition. Note that $\mnormSize{}{\hat\Theta-\trueT}\leq 2\alpha$ and $\text{rank}(\hat \Theta-\trueT)\leq 2\mr$. By Lemma~\ref{lem:Mnormbound}, $\anormSize{}{\hat \Theta-\trueT}\leq 2^{(3K-1)/2} \alpha \left({\prod r_k \over r_{\max}}\right)^{3/2}$. Therefore, the condition in Lemma~\ref{lem:convexity} holds with $\beta=2^{(3K-1)/2}\alpha \left({\prod r_k \over r_{\max}}\right)^{3/2}$.
Applying Lemma~\ref{lem:convexity} to~\eqref{eq:sample} gives
\begin{align}
 \PiFnormSize{}{\hat \Theta-\trueT}^2&\leq {1\over m}\MFnormSize{}{\hat \Theta-\trueT}^2+c\beta\sqrt{\sum_k d_k\over |\Omega|}\\
 &\leq {C_2  r^{K-1}_{\max}} {\sum_k d_k \over |\Omega|}+C_1 \alpha r_{\max}^{3(K-1)/2}\sqrt{\sum_kd_k\over |\Omega|},
\end{align}
with probability at least $1-\exp(-{\sum_kd_k\over \sum_k \log d_k})$ over the sampled set $\Omega$. Here $C_1, C_2>0$ are two constants independent of the tensor dimension and rank. Therefore,
\[
 \PiFnormSize{}{\hat \Theta-\trueT}^2\to 0,\quad \text{as}\quad {|\Omega|\over \sum_kd_k}\to \infty,
\]
provided that $r_{\max}=O(1)$.
\end{proof}

\subsection{Convexity of the log-likelihood function}\label{sec:proofconvexity}
\begin{thm}\label{thm:convexity}
Define the function
\begin{equation}\label{eq:function}
 \logl(\Theta, \mb)=\sum_{\omega\in\Omega}\sum_{\ell\in[L]} \big(\mathds{1}\{y_\omega=\ell\}\log \left[f(b_\ell-\theta_\omega)-  f(b_{\ell-1}-\theta_\omega)\right]\big),
 \end{equation}
where $f(\cdot)$ satisfies Assumption~\ref{ass:link}. Then, $\logl(\Theta, \mb)$ is concave in $(\Theta,\mb)$.
\end{thm}

\begin{proof}
Define $d_{\text{total}}=\prod_k d_k$. By abuse of notation, we use $(\Theta,\mb)$ to denote the length-$(d_{\text{total}}+L-1)$-vector collecting all parameters together. Let us denote a bivariate function
\begin{align}
\lambda\colon &\mathbb{R}^2\mapsto \mathbb{R}\\
(u,v)&\mapsto\lambda(u,v) = \log \big[f(u)-  f(v)\big].
\end{align}
It suffices to show that $\lambda(u,v)$ is concave in $(u,v)$ where $u>v$.

Suppose that the claim holds (which we will prove in the next paragraph). Based on~\eqref{eq:function}, $u,v$ are both linear functions of $(\Theta,\mb)$:
\[
u = \ma_1^T(\Theta,\mb), \quad v = \ma_2^T(\Theta,\mb), \quad \text{ for some}\ \ma_1,\ma_2\in \mathbb{R}^{d_{\text{total}}+L-1}.
\]
Then, $\lambda(u,v) = \lambda(\ma_1^T(\Theta,\mb),\ \ma_2^T(\Theta,\mb))$ is concave in $(\Theta,\mb)$ by the definition of concavity. Therefore, we can conclude that $\logl(\Theta, \mb) $ is concave in $(\Theta,\mb)$ because $\logl(\Theta, \mb)$ is the sum of $\lambda(u,v)$.

Now, we prove the concavity of $\lambda(u,v)$. Note that
\begin{equation}
\lambda(u,v) = \log\big[f(u)-f(v)\big]=\log\big[\int\mathds{1}_{[u,v]}(x)f'(x)dx\big],
\end{equation}
where $\mathds{1}_{[u,v]}$ is an indicator function that equals 1 in the interval $[u,v]$, and 0 elsewhere. Furthermore, $\mathds{1}_{[u,v]}(x)$ is log-concave in $(u,v,x)$, and by Assumption~\ref{ass:link}, $f'(x)$ is log-concave in $x$. It follows that $\mathds{1}_{[u,v]}(x)f'(x)$
is a log-concave in $(u,v,x)$. By Lemma~\ref{lem:lossconvexity}, we conclude that $\lambda(u,v)$ is concave in $(u,v)$ where $u>v$.
\end{proof}

\begin{lem}[Corollary 3.5 in~\citet{brascamp2002extensions}]\label{lem:lossconvexity}
Let $F(x,y)\colon \mathbb{R}^{m+n}\rightarrow \mathbb{R}$ be an integrable function where $x\in \mathbb{R}^{m},y\in \mathbb{R}^n$. Let
\[
G(x) = \int_{\mathbb{R}^n}F(x,y)dy.
\]
If $F(x,y)$ is log concave in $(x,y)$, then $G(x)$ is log concave in $x$.
\end{lem}

\section{Conclusions}
We have developed a low-rank tensor estimation method based on possibly incomplete, ordinal-valued observations. A sharp error bound is established, and we demonstrate the outperformance of our approach compared to other methods. The work unlocks several directions of future research. One interesting question would be the inference problem, i.e.,\ to assess the uncertainty of the obtained estimates and the imputation. Other directions include the trade-off between (non)convex optimization and statistical efficiency. While we have provided numerical evidence for the success of nonconvex approach, the full landscape of the optimization remains open. The interplay between computational efficiency and statistical accuracy in general tensor problems warrants future research.

\section*{Acknowledgements}
This research is supported in part by NSF grant DMS-1915978 and Wisconsin Alumni Research Foundation.

\clearpage
\appendix

\renewcommand{\thefigure}{{S\arabic{figure}}}%
\renewcommand{\figurename}{{Supplementary Figure}}    
\setcounter{figure}{0}   
\setcounter{table}{0}  

\section*{Appendix}

\section{Auxiliary lemmas}\label{sec:lemma}
This section collects lemmas that are useful for the proofs of the main theorems.

\begin{defn}[Atomic M-norm~\citep{ghadermarzy2019near}]
Define $T_{\pm}=\{\tT\in\{\pm 1 \}^{d_1\times \cdots \times d_K}\colon \text{rank}(\tT)=1\}$. The atomic M-norm of a tensor $\Theta\in\mathbb{R}^{d_1\times \cdots \times d_K}$ is defined as
\begin{align}
\anormSize{}{\Theta}&=\inf\{t>0\colon \Theta\in t\text{conv}(T_{\pm})\}\\
&=\inf\left\{\sum_{\tX\in T_{\pm}}c_{\tX} \colon\ \Theta=\sum_{\tX\in T_{\pm}} c_{\tX}\tX, \ c_{\tX}>0\right\}.
\end{align}
\end{defn}

\begin{defn}[Spectral norm~\citep{lim2005singular}]
The spectral norm of a tensor $\Theta\in\mathbb{R}^{d_1\times \cdots \times d_K}$ is defined as
\[
\snormSize{}{\Theta}=\sup\left\{\langle \Theta, \mx_1\otimes \cdots \otimes \mx_K\rangle \colon \vnormSize{}{\mx_k}=1,\ \mx_k\in\mathbb{R}^{d_k},\ \text{for all}\ k\in[K]\right\}.
\]
\end{defn}

\begin{defn}[Nuclear norm~\citep{friedland2018nuclear}]
The nuclear norm of a tensor $\Theta\in\mathbb{R}^{d_1\times \cdots \times d_K}$ is defined as
\[
\nnormSize{}{\Theta}=\inf
\left\{
\sum_{i\in[r]}|\lambda_i|\colon \Theta=\sum_{i=1}^r \lambda_i\mx^{(i)}_1\otimes \cdots \otimes \mx^{(i)}_K,\ \vnormSize{}{\mx^{(i)}_k}=1,\ \mx^{(i)}_k\in\mathbb{R}^{d_k},\ \text{for all}\ k\in[K],\ i\in[r]
\right\},
\]
where the infimum is taken over all $r\in\mathbb{N}$ and $\vnormSize{}{\mx^{(i)}_k}=1$ for all $i\in[r]$ and $k\in[K]$.
\end{defn}

\begin{lem}[M-norm and infinity norm~\citep{ghadermarzy2019near}]\label{lem:Mnormbound}
Let $\Theta\in\mathbb{R}^{d_1\times \cdots \times d_K}$ be an order-$K$, rank-$(r_1,\ldots,r_K)$ tensor. Then
\[
\mnormSize{}{\Theta}\leq \anormSize{}{\Theta}\leq \left(\prod_k r_k \over r_{\max}\right)^{3\over 2} \mnormSize{}{\Theta}.
\]
\end{lem}

\begin{lem}[Nuclear norm and F-norm] \label{lem:nuclear}
Let $\tA\in\mathbb{R}^{d_1\times\cdots\times d_K}$ be an order-$K$ tensor with Tucker $\text{rank}(\tA)=(r_1,\ldots,r_K)$. Then
\[
\nnormSize{}{\tA} \leq \sqrt{\prod_k r_k\over \max_k r_k} \FnormSize{}{\tA},
\]
where $\nnormSize{}{\cdot}$ denotes the nuclear norm of the tensor.
\end{lem}

\begin{proof}
Without loss of generality, suppose $r_1=\min_k r_k$. Let $\tA_{(k)}$ denote the mode-$k$ matricization of $\tA$ for all $k\in[K]$. By \citet[Corollary 4.11]{wang2017operator}, and the invariance relationship between a tensor and its Tucker core~\citep[Section 6]{jiang2017tensor}, we have
\begin{equation}\label{eq:norminequality}
\nnormSize{}{\tA} \leq \sqrt{\prod_{k\geq 2} r_k \over \max_{k\geq 2} r_k} \nnormSize{}{\tA_{(1)}},
\end{equation}
where $\tA_{(1)}$ is a $d_1$-by-$\prod_{k\geq 2}d_k$ matrix with matrix rank $r_1$. Furthermore, the relationship between the matrix norms implies that $\nnormSize{}{\tA_{(1)}}\leq \sqrt{r_1}\FnormSize{}{\tA_{(1)}}=\sqrt{r_1}\FnormSize{}{\tA}$. Combining this fact with the inequality~\eqref{eq:norminequality} yields the final claim.
\end{proof}

\begin{lem} \label{lem:inq}
Let $\tA, \; \tB$ be two order-$K$ tensors of the same dimension. Then
\[
|\langle \tA,\tB\rangle| \leq \snormSize{}{\tA}   \nnormSize{}{\tB}.
\]
\end{lem}

\begin{proof}
By~\citet[Proposition 3.1]{friedland2018nuclear}, there exists a nuclear norm decomposition of $\tB$, such that
\[
\tB=\sum_{r} \lambda_r \ma^{(1)}_r\otimes \cdots\otimes \ma^{(K)}_r,\quad \ma_r^{(k)}\in\mathbf{S}^{d_k-1}(\mathbb{R}),\quad \text{for all }k\in[K],
\]
and $\nnormSize{}{\tB}=\sum_{r}|\lambda_r|$. Henceforth we have
\begin{align*}
|\langle \tA,\tB\rangle|&=| \langle \tA, \sum_{r} \lambda_r \ma^{(1)}_r\otimes \cdots\otimes \ma^{(K)}_r \rangle|\leq \sum_r |\lambda_r| |\langle \tA, \ma^{(1)}_r \otimes \cdots\otimes \ma^{(K)}_r \rangle|\\
&\leq \sum_{r}|\lambda_r| \snormSize{}{\tA}= \snormSize{}{\tA}\nnormSize{}{\tB},
\end{align*}
which completes the proof.
\end{proof}

The following lemma provides the bound on the spectral norm of random tensors. The result was firstly presented in~\citet{nguyen2015tensor}, and we adopt the version from~\citet{tomioka2014spectral}.
\begin{lem}[Spectral norm of random tensors~\citep{tomioka2014spectral}]\label{lem:tensor}
Suppose that $\tS=\entry{s_{\omega}}\in\mathbb{R}^{d_1\times \cdots \times d_K}$ is an order-$K$ tensor whose entries are independent random variables that satisfy
\[
\mathbb{E}(s_{\omega})=0,\quad \text{and} \quad\mathbb{E}(e^{ts_{\omega}})\leq e^{t^2L^2/2}.
\]
Then the spectral norm $\snormSize{}{\tS}$ satisfies that,
\[
\snormSize{}{\tS}\leq \sqrt{{8L^2} \log (12K) \sum_k d_k +\log (2/\delta)},
\]
with probability at least $1-\delta$.
\end{lem}

\begin{lem} \label{lem:noisytensor}
Suppose that $\tS=\entry{s_{\omega}}\in\mathbb{R}^{d_1\times \cdots \times d_K}$ is an order-$K$ tensor whose entries are independent random variables that satisfy
\[
\mathbb{E}(s_{\omega})=0,\quad \text{and}\quad |s_{\omega}|\leq U.
\]
Then we have
\[
\mathbb{P}\left(\snormSize{}{\tS}\geq C_2 U\sqrt{\sum_k d_k} \right)\leq \exp\left(-C_1  \log K \sum_k d_k\right)
\]
where $C_1>0$ is an absolute constant, and $C_2>0$ is a constant that depends only on $K$.
\end{lem}

\begin{proof}  Note that the random variable $U^{-1}s_{\omega}$ is zero-mean and supported on $[-1,1]$. Therefore, $U^{-1}s_{\omega}$ is sub-Gaussian with parameter ${1-(-1)\over 2}=1$, i.e.
\[
\mathbb{E}(U^{-1}s_{\omega})=0,\quad \text{and}\quad \mathbb{E}(e^{tU^{-1}s_{\omega}})\leq e^{t^2/2}.
\]
It follows from Lemma~\ref{lem:tensor} that, with probability at least $1-\delta$,
\[
\snormSize{}{U^{-1}\tS}\leq \sqrt{\left(c_0\log K+c_1\right) \sum_k d_k +\log (2/\delta)},
\]
where $c_0, c_1>0$ are two absolute constants. Taking $\delta=\exp (-C_1\log K \sum_k d_k)$ yields the final claim, where $C_2=c_0\log K+c_1+1>0$ is another constant.
\end{proof}

\begin{lem}\label{lem:KLentry} Let $X,\; Y$ be two discrete random variables taking values on $L$ possible categories, with point mass probabilities $\{p_\ell\}_{\ell\in[L]}$ and $\{q_\ell\}_{\ell\in[L]}$, respectively.  Suppose $p_\ell$, $q_\ell>0$ for all $\ell\in[L]$. Then, the Kullback-Leibler (KL) divergence satisfies that
\[
\mathrm{KL}(X||Y)\stackrel{\text{def}}{=}-\sum_{\ell\in[L]}\mathbb{P}_X(\ell)\log\left\{{\mathbb{P}_Y(\ell)\over \mathbb{P}_X(\ell)}\right\} \leq \sum_{\ell \in [L]}{(p_\ell-q_\ell)^2 \over q_\ell}.
\]
\end{lem}
\begin{proof} Using the fact $\log x\leq x-1$ for $x>0$, we have that
\begin{align}\label{eq:KL}
\text{KL}(X||Y)&=\sum_{\ell\in[L]}p_\ell\log{p_\ell\over q_\ell}\\
&\leq \sum_{\ell\in[L]} {p_\ell\over q_\ell}(p_\ell-q_\ell)\\
&=\sum_{\ell\in[L]} \left({p_\ell\over q_\ell }- 1\right)(p_\ell-q_\ell)+ \sum_{\ell\in[L]} (p_\ell-q_\ell).
\end{align}
Note that $\sum_{\ell\in[L]}(p_\ell-q_\ell)=0$. Therefore,
\[
\text{KL}(X||Y)\leq \sum_{\ell\in[L]}\left( {p_\ell\over q_\ell}-1\right)\left(p_\ell-q_\ell\right)=\sum_{\ell\in[L]}{(p_\ell-q_\ell)^2\over q_\ell}.
\]
\end{proof}

\begin{lem}[KL divergence and F-norm]~\label{lem:KL}
Let $\tY\in[L]^{d_1\times \cdots \times d_K}$ be an ordinal tensor generated from the model~\eqref{eq:model} with the link function $f$ and parameter tensor $\Theta$. Let $\mathbb{P}_{\Theta}$ denote the joint categorical distribution of $\tY|\Theta$ induced by the parameter tensor $\Theta$, where $\mnormSize{}{\Theta}\leq \alpha$. Define
\begin{equation}\label{eq:ass}
A_\alpha=\min_{\ell\in[L], |\theta|\leq \alpha}\left[f(b_\ell-\theta)-f(b_{\ell-1}-\theta)\right].
\end{equation}
Then, for any two tensors $\Theta,\; \Theta^*$ in the parameter spaces, we have
\[
\mathrm{KL}(\mathbb{P}_{\Theta}|| \mathbb{P}_{\Theta^*})\leq {2(2L-3)\over A_\alpha} \dot{f}^2(0)\FnormSize{}{\Theta-\Theta^*}^2.
\]
\end{lem}
\begin{proof} Suppose that the distribution over the ordinal tensor $\tY=\entry{y_\omega}$ is induced by $\Theta=\entry{\theta_\omega}$. Then, based on the generative model~\eqref{eq:model},
\[
\mathbb{P}(y_\omega=\ell | \theta_\omega)=f(b_{\ell}-\theta_\omega)-f(b_{\ell-1}-\theta_\omega),
\]
for all $\ell\in[L]$ and $\omega\in[d_1]\times \cdots \times [d_K]$. For notational convenience, we suppress the subscribe in $\theta_\omega$ and simply write $\theta$ (and respectively, $\theta^*$). Based on Lemma~\ref{lem:KLentry} and Taylor expansion,
\begin{align}
\text{KL}(\theta|| \theta^*) & \leq \sum_{\ell\in[L]} {\left[f(b_{\ell}-\theta)-f(b_{\ell-1}-\theta)-f(b_{\ell}-\theta^*)+f(b_{\ell-1}-\theta^*)\right]^2\over f(b_{\ell}-\theta^*)-f(b_{\ell-1}-\theta^*)}\\
 &\leq \sum_{\ell=2}^{L-1} {\left[\dot{f}(b_\ell-\eta_\ell)-\dot{f}(b_{\ell-1}-\eta_{\ell-1})\right]^2 \over f(b_\ell-\theta^*)-f(b_{\ell-1}-\theta^*)} \left(\theta-\theta^*\right)^2+{\dot{f}^2(b_1-\eta_1) \over f(b_1-\theta^*)} (\theta-\theta^*)^2\\
 &\quad \quad\quad \quad  +{\dot{f}^2(b_{L-1}-\eta_{L-1})\over 1-f(b_{L-1}-\theta^*)}(\theta-\theta^*)^2,
\end{align}
where $\eta_\ell$ and $\eta_{\ell-1}$ fall between $\theta$ and $\theta^*$. Therefore,
\begin{equation}\label{eq:entrywise}
\text{KL}(\theta|| \theta^*) \leq \left({4(L-2)\over A_\alpha}+ {2\over A_\alpha}\right)\dot{f}^2(0)(\theta-\theta^*)^2={2(2L-3)\over A_\alpha} \dot{f}^2(0)(\theta-\theta^*)^2,
\end{equation}
where we have used Taylor expansion, the bound~\eqref{eq:ass}, and the fact that $\dot{f}(\cdot)$ peaks at zero for an unimodal and symmetric function. Now summing~\eqref{eq:entrywise} over the index set $\omega\in[d_1]\times \cdots \times [d_K]$ gives
\[
\text{KL}(\mathbb{P}_{\Theta}|| \mathbb{P}_{\Theta^*}) =\sum_{\omega\in[d_1]\times \cdots \times[d_K]} \text{KL}(\theta_\omega || \theta^*_\omega) \leq {2(2L-3)\over A_\alpha} \dot{f}^2(0)\FnormSize{}{\Theta-\Theta^*}^2.
\]
\end{proof}

\begin{rmk} In particular, let $\mathbb{P}_{\bf{0}}$ denote the distribution of $\tY|\bf{0}$ induced by the zero parameter tensor. Then we have
\[
\text{KL}(\mathbb{P}_{\Theta}||\mathbb{P}_{\bf{0}})\leq {2(2L-3)\over A_\alpha}  \dot{f}^2(0)\FnormSize{}{\Theta}^2.
\]
\end{rmk}

\begin{lem}\label{lem:construction}
Assume the same setup as in Theorem~\ref{thm:minimax}. Without loss of generality, suppose $d_1=\max_kd_k$. Define $R=\max_k r_k$ and $d_{\text{total}}=\prod_{k\in[K]} d_k$. For any constant $0\leq \gamma \leq 1$, there exist a finite set of tensors $\tX=\{\Theta_i: i=1,\ldots\}\subset \tP$ satisfying the following four properties:
\begin{enumerate}[label=\textup{(\roman*)}]
\item $\text{Card}(\tX)\geq 2^{Rd_1/8}+1$, where $\text{Card}$ denotes the cardinality;
\item $\tX$ contains the zero tensor $\mathbf{0}\in\mathbb{R}^{d_1\times \cdots\times d_K}$;
\item $\mnormSize{}{\Theta}\leq \gamma \min\left\{ \alpha ,\ C^{-1/2}\sqrt{Rd_1\over d_{\text{total}}} \right\} $ for any element $\Theta\in\tX$;
\item $\FnormSize{}{\Theta_i-\Theta_j}\geq {\gamma\over 4} \min\left\{ \alpha\sqrt{d_{\text{total}}},\ C^{-1/2}\sqrt{Rd_1}\right\}$ for any two distinct elements $\Theta_i,\; \Theta_j\in\tX$,
\end{enumerate}
Here $C=C(\alpha,L,f,\mb)={(4L-6)\dot{f}^2(0)\over A_\alpha }>0$ is a constant independent of the tensor dimension and rank.
\end{lem}

\begin{proof}
Given a constant $0\leq \gamma \leq 1$, we define a set of matrices:
\[
\tC=\left\{\mM=(m_{ij})\in\mathbb{R}^{d_1\times R}: a_{ij}\in \left\{ 0,\gamma \min\left\{ \alpha , C^{-1/2}\sqrt{Rd_1\over d_{\text{total}}} \right\}\right\} ,\  \forall (i,j)\in[d_1]\times[R]\right\}.
\]
We then consider the associated set of block tensors:
\begin{align}
\tB=\tB(\tC)=\big\{\Theta\in\mathbb{R}^{d_1\times \cdots \times d_K}\colon& \Theta=\mA\otimes \mathbf{1}_{d_3}\otimes \cdots \otimes \mathbf{1}_{d_K}, \\
& \text{where}\ \mA=(\mM|\cdots|\mM|\mO) \in\mathbb{R}^{d_1\times d_2},\ \mM\in\tC\big\},
\end{align}
where $\mathbf{1}_d$ denotes a length-$d$ vector with all entries 1, $\mO$ denotes the $d_1\times (d_2-R\lfloor d_2/R \rfloor)$ zero matrix, and $\lfloor d_2/ R \rfloor$ is the integer part of $d_2/R$. In other words, the subtensor $\Theta(\mI, \mI,i_3, \ldots,i_K)\in\mathbb{R}^{d_1\times d_2}$ are the same for all fixed $(i_3,\ldots,i_K)\in[d_3]\times \cdots \times [d_K]$, and furthermore, each subtensor $\Theta(\mI,\mI, i_3,\ldots,i_K)$ itself is filled by copying the matrix $\mM\in\mathbb{R}^{d_1\times R}$ as many times as would fit.

By construction, any element of $\tB$, as well as the difference of any two elements of $\tB$, has Tucker rank at most $\max_k r_k\leq R$, and the entries of any tensor in $\tB$ take values in $[0,\alpha]$. Thus, $\tB\subset\tP$. By Lemma~\ref{lem:VGbound}, there exists a subset $\tX\subset \tB$ with cardinality $\text{Card}(\tX)\geq 2^{Rd_1/8}+1$ containing the zero $d_1\times \cdots \times d_K$ tensor, such that, for any two distinct elements $\Theta_i$ and $\Theta_j$ in $\tX$,
\[
\FnormSize{}{\Theta_i-\Theta_j}^2 \geq {Rd_1\over 8} \gamma^2\min\left\{ \alpha^2, {C^{-1}Rd_1 \over d_{\text{total}}}\right\} \Big\lfloor {d_2\over R}\Big \rfloor \prod_{k\geq 3}d_k\geq {\gamma^2\min\left\{ \alpha^2 d_{\text{total}}, C^{-1}Rd_1\right\}  \over 16}.
\]
In addition, each entry of $\Theta\in\tX$ is bounded by $\gamma \min\left\{ \alpha , C^{-1/2}\sqrt{Rd_1\over d_{\text{total}}}\right\} $. Therefore the Properties (i) to (iv) are satisfied.
\end{proof}

\begin{lem}[Varshamov-Gilbert bound]\label{lem:VGbound}
Let $\Omega=\{(w_1,\ldots,w_m)\colon w_i\in\{0,1\}\}$. Suppose $m>8$. Then there exists a subset $\{w^{(0)},\ldots,w^{(M)}\}$ of $\Omega$ such that $w^{(0)}=(0,\ldots,0)$ and
\[
\zeronormSize{}{w^{(j)}-w^{(k)}}\geq {m\over 8},\quad \text{for } \ 0\leq j<k\leq M,
\]
where $\zeronormSize{}{\cdot}$ denotes the Hamming distance, and $M\geq 2^{m/8}$.
\end{lem}

\begin{lem}[Theorem 2.5 in~\citet{tsybakov2008introduction}]\label{lem:Tsybakov}
Assume that a set $\tX$ contains element $\Theta_0, \Theta_1, \ldots,\Theta_M$ ($M\geq 2$) such that
\begin{enumerate}[label=\textup{(\roman*)}]
\item $d(\Theta_j,\ \Theta_k)\geq 2s>0$, $\forall 0\leq j\leq k\leq M$;
\item $\mathbb{P}_0$ is absolutely continuous with respect to $\mathbb{P}_j$, $\forall j=1,\ldots,M$, and
\[
{1\over M}\sum_{j=1}^M \mathrm{KL}(\mathbb{P}_j||\mathbb{P}_0)\leq \alpha \log M
\]
where $d\colon \tX\times \tX\mapsto [0,+\infty]$ is a semi-distance function, $0<\alpha<{1/8}$ and $\mathbb{P}_j=\mathbb{P}_{\Theta_j}$, $j=0,1\ldots,M$.
\end{enumerate}
Then
\[
\inf_{\hat \Theta}\sup_{\Theta\in\tX} \mathbb{P}_{\Theta}(d(\hat \Theta, \Theta)\geq s)\geq {\sqrt{M}\over 1+\sqrt{M}}\left(1-2\alpha-\sqrt{2\alpha\over \log M} \right)>0.
\]
\end{lem}

\begin{lem}[Lemma 28 in~\citet{ghadermarzy2019near}]\label{lem:convexity}
Define $\mathbb{B}_{M}(\beta)=\{\Theta\in \mathbb{R}^{d_1\times \cdots \times d_K}\colon \anormSize{}{\Theta}\leq \beta \}$.  Let $\Omega\subset[d_1]\times\cdots \times [d_K]$ be a random set with $m=|\Omega|$, and assume that each entry in $\Omega$ is drawn with replacement from $[d_1]\times\cdots\times[d_K]$ using probability $\Pi$. Define
\[
\PiFnormSize{}{\Theta}^2={1\over m}\mathbb{E}_{\Omega\in\Pi}\MFnormSize{}{\Theta}^2.
\]
Then, there exists a universal constant $c>0$, such that, with probability at least $1-\exp\left(-{\sum_kd_k \over \sum_k \log d_k} \right)$ over the sampled set $\Omega$,
\[
{1\over m}\MFnormSize{}{\Theta}^2 \geq \PiFnormSize{}{\Theta}^2-c\beta\sqrt{\sum_k d_k\over m}
\]
holds uniformly for all $\Theta\in\mathbb{B}_M(\beta)$.
\end{lem}

\section{Additional explanations of HCP analysis}\label{sec:additionalHCP}
We perform clustering analyses based on the Tucker representation of the estimated signal tensor $\hat\Theta$. The procedure is motivated from the higher-order extension of Principal Component Analysis (PCA) or Singular Value Decomposition (SVD). Recall that, in the matrix case, we perform clustering on an $m\times n$ (normalized) matrix $\mX$ based on the following procedure. First, we factorize $X$ into
\begin{equation}
\mX = \mU\mSigma \mV^T,
\end{equation}
where $\mSigma$ is a diagonal matrix and $\mU,\mV$ are factor matrices with orthogonal columns. Second, we take each column of $\mV$ as a principal axis and each row in $\mU\mSigma$ as principal component. A subsequent multivariate clustering method (such as $K$-means) is then applied to the $m$ rows of $\mU\mSigma$.

We apply a similar clustering procedure to the estimated signal tensor $\hat\Theta$. We factorize $\hat \Theta$ based on Tucker decomposition.
\begin{equation}\label{eq:Tuckerest}
\hat \Theta = \hat \tC\times_1\hat \mM_1\times_2\cdots\times_K\hat \mM_K,
\end{equation}
where $\hat \tC\in\mathbb{R}^{r_1\times \cdots \times r_K}$ is the estimated core tensor, $\hat \mM_k\in\mathbb{R}^{d_k\times r_k}$ are estimated factor matrices with orthogonal columns, and $\times_k$ denotes the tensor-by-matrix multiplication~\citep{kolda2009tensor}. The mode-$k$ matricization of~\eqref{eq:Tuckerest} gives
\begin{equation}
\hat \Theta_{(k)} = \hat \mM_k\hat \tC_{(k)}\left(\hat \mM_K\otimes\cdots\otimes\hat \mM_1\right),
\end{equation}
where $\hat \Theta_{(k)} , \hat \tC_{(k)}$ denote the mode-$k$ unfolding of $\hat \Theta$ and $\hat \tC$, respectively. We conduct clustering on this  the mode-$k$ unfolded signal tensor. We take columns in $\left(\hat \mM_K\otimes\cdots\otimes\hat \mM_1\right)$ as principal axes and rows in $\hat \mM_k\hat \tC_{(k)}$ as principal components. Then, we apply $K$-means clustering method to the $d_k$ rows of the matrix $\hat \mM_k\hat \tC_{(k)}$.

We perform a clustering analysis on the 68 brain nodes using the procedure described above. Our ordinal tensor method outputs the estimated parameter tensor $\hat\Theta\in\mathbb{R}^{68\times 68\times136}$ with rank $(23,23,8)$. We apply $K$-means to the mode-1 principal component matrix of size $68\times 184$ ($184=23\times8$). The elbow method suggests 11 clusters among the 68 nodes (see Figure~\ref{figure:elbow}).
The clustering result is presented in Section \ref{sec:dataapplication}.

\begin{figure}[ht]
\begin{center}
\includegraphics[width=0.4\textwidth]{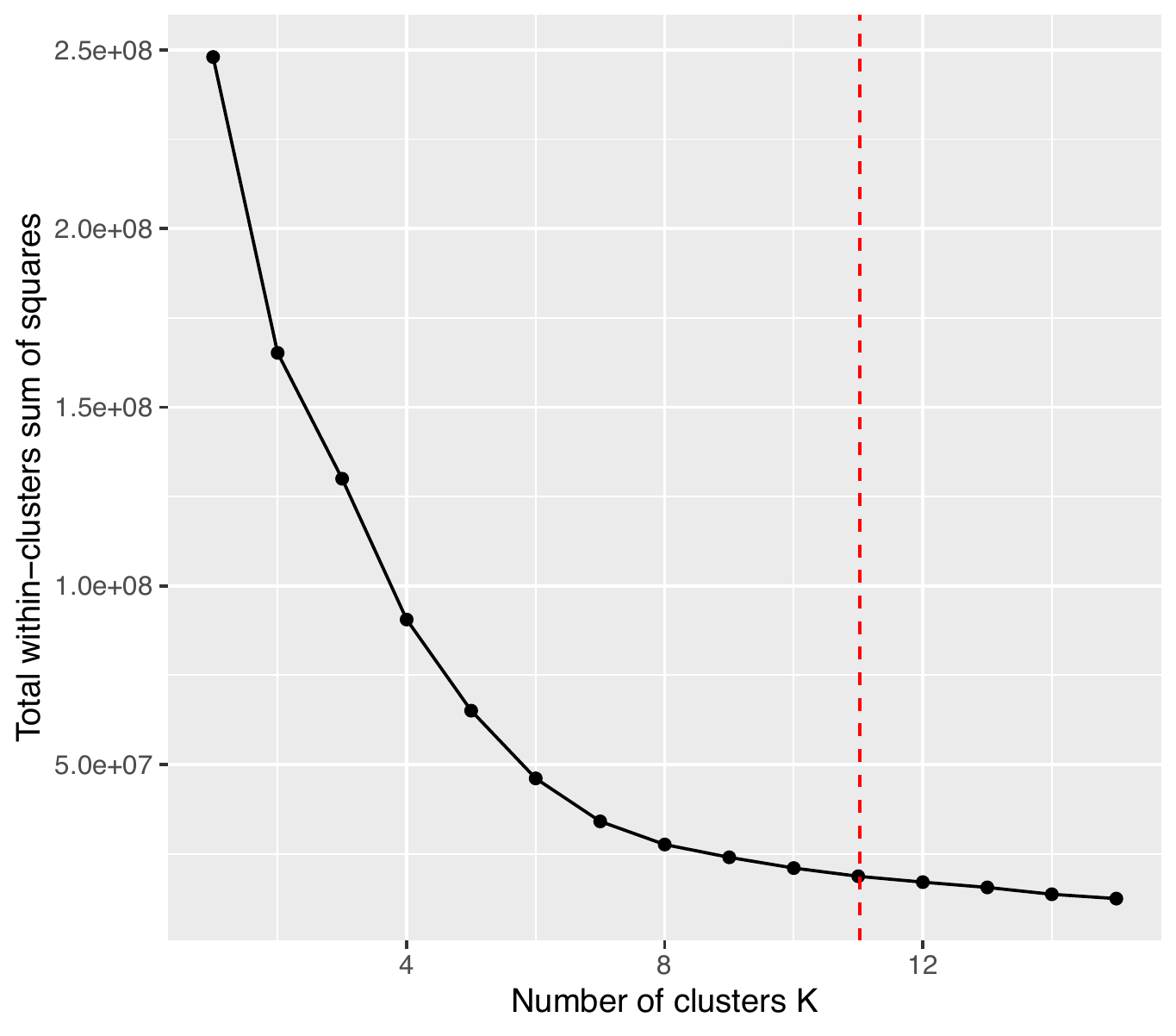}
\end{center}
\caption{Elbow plot for determining the number of clusters in $K$-means.}\label{figure:elbow}
\end{figure}

\clearpage
\bibliography{tensor_wang}
\bibliographystyle{apalike}

\end{document}

%% file: macros.tex
\usepackage{stmaryrd} 

\def\ma{\bm{a}}
\def\mb{\bm{b}}

\def\mr{\bm{r}}

\def\mx{\bm{x}}

\def\mz{\bm{z}}

\def\mA{\bm{A}}
\def\mB{\bm{B}}

\def\mH{\bm{H}}
\def\mI{\bm{I}}

\def\mI{\bm{I}}
\def\mM{\bm{M}}

\def\mO{\bm{O}}

\def\mU{\bm{U}}
\def\mV{\bm{V}}

\def\mX{\bm{X}}

\def\mx{\bm x}

\def\mA{\bm A}
\def\mB{\bm B}

\def\mI{\bm I}

\def\mX{\bm X}

\def\mM{\bm M}
\def\mSigma{\bm \Sigma}

\def\tZ{\mathcal{Z}}
\def\tA{\mathcal{A}}
\def\tB{\mathcal{B}}
\def\tC{\mathcal{C}}

\def\tE{\mathcal{E}}

\def\tJ{\mathcal{J}}

\def\tL{\mathcal{L}}

\def\tO{\mathcal{O}}
\def\tP{\mathcal{P}}

\def\tS{\mathcal{S}}
\def\tT{\mathcal{T}}

\def\tX{\mathcal{X}}
\def\tY{\mathcal{Y}}
\def\tZ{\mathcal{Z}}
\def\trueT{\Theta^{\mathrm{true}}}
\def\trueb{\mb^{\mathrm{true}}}
\def\trueZ{\tZ^{\mathrm{true}}}
\def\plogZ{\nabla_{\tZ}\tL_{\tY}}
\def\pplogZ{\nabla^2_{\tZ}\tL_{\tY}}
\def\logl{\tL_{\tY,\Omega}}
\def\plogl{\nabla_{\Theta}\tL_{\tY,\Omega}}
\def\pplogl{\nabla^2_{\Theta}\tL_{\tY,\Omega}}
\def\flogl{\tL_{\tY}}
\def\fplogl{\nabla_{\Theta}\tL_{\tY}}
\def\fpplogl{\nabla^2_{\Theta}\tL_{\tY}}

\newcommand{\MFnormSize}[2]{#1\lVert#2#1\rVert_{F,\Omega}}

\def\entry#1{\llbracket #1 \rrbracket}
\def\entry#1{\llbracket #1 \rrbracket}

\def\Vec{\operatorname{Vec}}

\newcommand{\zeronormSize}[2]{#1\lVert#2#1\rVert_0}
\newcommand{\nnormSize}[2]{#1\lVert#2#1\rVert_*}
\newcommand{\anormSize}[2]{#1\lVert#2#1\rVert_M}
\newcommand{\vnormSize}[2]{#1\lVert#2#1\rVert_2}
\newcommand{\snormSize}[2]{#1\lVert#2#1\rVert_\sigma}
\newcommand{\FnormSize}[2]{#1\lVert#2#1\rVert_F} 
\newcommand{\mnormSize}[2]{#1\lVert#2#1\rVert_\infty}

\newcommand{\PiFnormSize}[2]{#1\lVert#2#1\rVert_{F,\Pi}} 

\DeclareMathOperator*{\argmin}{arg\,min}
\DeclareMathOperator*{\argmax}{arg\,max}

%% file: ordinal.bbl
\begin{thebibliography}{}

\bibitem[Acar et~al., 2010]{acar2010scalable}
Acar, E., Dunlavy, D.~M., Kolda, T.~G., and M{\o}rup, M. (2010).
\newblock Scalable tensor factorizations with missing data.
\newblock In {\em Proceedings of the 2010 SIAM international conference on data
  mining}, pages 701--712. SIAM.

\bibitem[Baltrunas et~al., 2011]{baltrunas2011incarmusic}
Baltrunas, L., Kaminskas, M., Ludwig, B., Moling, O., Ricci, F., Aydin, A.,
  L{\"u}ke, K.-H., and Schwaiger, R. (2011).
\newblock Incarmusic: Context-aware music recommendations in a car.
\newblock In {\em International Conference on Electronic Commerce and Web
  Technologies}, pages 89--100. Springer.

\bibitem[Bhaskar, 2016]{bhaskar2016probabilistic}
Bhaskar, S.~A. (2016).
\newblock Probabilistic low-rank matrix completion from quantized measurements.
\newblock {\em The Journal of Machine Learning Research}, 17(1):2131--2164.

\bibitem[Bhaskar and Javanmard, 2015]{bhaskar20151}
Bhaskar, S.~A. and Javanmard, A. (2015).
\newblock 1-bit matrix completion under exact low-rank constraint.
\newblock In {\em 2015 49th Annual Conference on Information Sciences and
  Systems (CISS)}, pages 1--6. IEEE.

\bibitem[Brascamp and Lieb, 2002]{brascamp2002extensions}
Brascamp, H.~J. and Lieb, E.~H. (2002).
\newblock On extensions of the brunn-minkowski and pr{\'e}kopa-leindler
  theorems, including inequalities for log concave functions, and with an
  application to the diffusion equation.
\newblock In {\em Inequalities}, pages 441--464. Springer.

\bibitem[Cai and Zhou, 2013]{cai2013max}
Cai, T. and Zhou, W.-X. (2013).
\newblock A max-norm constrained minimization approach to 1-bit matrix
  completion.
\newblock {\em The Journal of Machine Learning Research}, 14(1):3619--3647.

\bibitem[Davenport et~al., 2014]{davenport2014}
Davenport, M.~A., Plan, Y., Van Den~Berg, E., and Wootters, M. (2014).
\newblock 1-bit matrix completion.
\newblock {\em Information and Inference: A Journal of the IMA}, 3(3):189--223.

\bibitem[De~Silva and Lim, 2008]{de2008tensor}
De~Silva, V. and Lim, L.-H. (2008).
\newblock Tensor rank and the ill-posedness of the best low-rank approximation
  problem.
\newblock {\em SIAM Journal on Matrix Analysis and Applications},
  30(3):1084--1127.

\bibitem[Friedland and Lim, 2018]{friedland2018nuclear}
Friedland, S. and Lim, L.-H. (2018).
\newblock Nuclear norm of higher-order tensors.
\newblock {\em Mathematics of Computation}, 87(311):1255--1281.

\bibitem[Ghadermarzy et~al., 2018]{ghadermarzy2018learning}
Ghadermarzy, N., Plan, Y., and Yilmaz, O. (2018).
\newblock Learning tensors from partial binary measurements.
\newblock {\em IEEE Transactions on Signal Processing}, 67(1):29--40.

\bibitem[Ghadermarzy et~al., 2019]{ghadermarzy2019near}
Ghadermarzy, N., Plan, Y., and Yilmaz, {\"O}. (2019).
\newblock Near-optimal sample complexity for convex tensor completion.
\newblock {\em Information and Inference: A Journal of the IMA}, 8(3):577--619.

\bibitem[Hillar and Lim, 2013]{hillar2013most}
Hillar, C.~J. and Lim, L.-H. (2013).
\newblock Most tensor problems are {NP}-hard.
\newblock {\em Journal of the ACM (JACM)}, 60(6):45.

\bibitem[Hitchcock, 1927]{hitchcock1927expression}
Hitchcock, F.~L. (1927).
\newblock The expression of a tensor or a polyadic as a sum of products.
\newblock {\em Journal of Mathematics and Physics}, 6(1-4):164--189.

\bibitem[Hong et~al., 2020]{hong2020generalized}
Hong, D., Kolda, T.~G., and Duersch, J.~A. (2020).
\newblock Generalized canonical polyadic tensor decomposition.
\newblock {\em SIAM Review}, 62(1):133--163.

\bibitem[Jiang et~al., 2017]{jiang2017tensor}
Jiang, B., Yang, F., and Zhang, S. (2017).
\newblock Tensor and its {Tucker} core: the invariance relationships.
\newblock {\em Numerical Linear Algebra with Applications}, 24(3):e2086.

\bibitem[Kolda and Bader, 2009]{kolda2009tensor}
Kolda, T.~G. and Bader, B.~W. (2009).
\newblock Tensor decompositions and applications.
\newblock {\em SIAM Review}, 51(3):455--500.

\bibitem[Lim, 2005]{lim2005singular}
Lim, L.-H. (2005).
\newblock Singular values and eigenvalues of tensors: a variational approach.
\newblock In {\em 1st IEEE International Workshop on Computational Advances in
  Multi-Sensor Adaptive Processing, 2005.}, pages 129--132. IEEE.

\bibitem[McCullagh, 1980]{mccullagh1980regression}
McCullagh, P. (1980).
\newblock Regression models for ordinal data.
\newblock {\em Journal of the Royal Statistical Society: Series B
  (Methodological)}, 42(2):109--127.

\bibitem[Montanari and Sun, 2018]{montanari2018spectral}
Montanari, A. and Sun, N. (2018).
\newblock Spectral algorithms for tensor completion.
\newblock {\em Communications on Pure and Applied Mathematics},
  71(11):2381--2425.

\bibitem[Mu et~al., 2014]{mu2014square}
Mu, C., Huang, B., Wright, J., and Goldfarb, D. (2014).
\newblock Square deal: Lower bounds and improved relaxations for tensor
  recovery.
\newblock In {\em International Conference on Machine Learning}, pages 73--81.

\bibitem[Negahban et~al., 2011]{negahban2011estimation}
Negahban, S., Wainwright, M.~J., et~al. (2011).
\newblock Estimation of (near) low-rank matrices with noise and
  high-dimensional scaling.
\newblock {\em The Annals of Statistics}, 39(2):1069--1097.

\bibitem[Nguyen et~al., 2015]{nguyen2015tensor}
Nguyen, N.~H., Drineas, P., and Tran, T.~D. (2015).
\newblock Tensor sparsification via a bound on the spectral norm of random
  tensors.
\newblock {\em Information and Inference: A Journal of the IMA}, 4(3):195--229.

\bibitem[Nickel et~al., 2011]{nickel2011three}
Nickel, M., Tresp, V., and Kriegel, H.-P. (2011).
\newblock A three-way model for collective learning on multi-relational data.
\newblock In {\em International Conference on Machine Learning}, volume~11,
  pages 809--816.

\bibitem[Oseledets, 2011]{oseledets2011tensor}
Oseledets, I.~V. (2011).
\newblock Tensor-train decomposition.
\newblock {\em SIAM Journal on Scientific Computing}, 33(5):2295--2317.

\bibitem[Stoeckel et~al., 2009]{stoeckel2009supramarginal}
Stoeckel, C., Gough, P.~M., Watkins, K.~E., and Devlin, J.~T. (2009).
\newblock Supramarginal gyrus involvement in visual word recognition.
\newblock {\em Cortex}, 45(9):1091--1096.

\bibitem[Sur and Cand{\`e}s, 2019]{sur2019modern}
Sur, P. and Cand{\`e}s, E.~J. (2019).
\newblock A modern maximum-likelihood theory for high-dimensional logistic
  regression.
\newblock {\em Proceedings of the National Academy of Sciences},
  116(29):14516--14525.

\bibitem[Tomioka and Suzuki, 2014]{tomioka2014spectral}
Tomioka, R. and Suzuki, T. (2014).
\newblock Spectral norm of random tensors.
\newblock {\em arXiv preprint arXiv:1407.1870}.

\bibitem[Tsybakov, 2008]{tsybakov2008introduction}
Tsybakov, A.~B. (2008).
\newblock {\em Introduction to nonparametric estimation}.
\newblock Springer Science \& Business Media.

\bibitem[Tucker, 1966]{tucker1966some}
Tucker, L.~R. (1966).
\newblock Some mathematical notes on three-mode factor analysis.
\newblock {\em Psychometrika}, 31(3):279--311.

\bibitem[Van~Essen et~al., 2013]{van2013wu}
Van~Essen, D.~C., Smith, S.~M., Barch, D.~M., Behrens, T.~E., Yacoub, E.,
  Ugurbil, K., Consortium, W.-M.~H., et~al. (2013).
\newblock The {WU-Minn} human connectome project: an overview.
\newblock {\em Neuroimage}, 80:62--79.

\bibitem[Wang et~al., 2017]{wang2017operator}
Wang, M., Duc, K.~D., Fischer, J., and Song, Y.~S. (2017).
\newblock Operator norm inequalities between tensor unfoldings on the partition
  lattice.
\newblock {\em Linear Algebra and Its Applications}, 520:44--66.

\bibitem[Wang et~al., 2019]{wang2019three}
Wang, M., Fischer, J., and Song, Y.~S. (2019).
\newblock Three-way clustering of multi-tissue multi-individual gene expression
  data using semi-nonnegative tensor decomposition.
\newblock {\em The Annals of Applied Statistics}, 13(2):1103--1127.

\bibitem[Wang and Li, 2020]{wang2020learning}
Wang, M. and Li, L. (2020).
\newblock Learning from binary multiway data: Probabilistic tensor
  decomposition and its statistical optimality.
\newblock {\em Journal of Machine Learning Research}, 21(154):1--38.

\bibitem[Wang and Song, 2017]{wang2017tensor}
Wang, M. and Song, Y. (2017).
\newblock Tensor decompositions via two-mode higher-order {SVD} ({HOSVD}).
\newblock In {\em Proceedings of the 20th International Conference on
  Artificial Intelligence and Statistics}, pages 614--622.

\bibitem[Wang and Zeng, 2019]{wang2019multiway}
Wang, M. and Zeng, Y. (2019).
\newblock Multiway clustering via tensor block models.
\newblock In {\em Advances in Neural Information Processing Systems}, pages
  713--723.

\bibitem[Yuan and Zhang, 2016]{yuan2016tensor}
Yuan, M. and Zhang, C.-H. (2016).
\newblock On tensor completion via nuclear norm minimization.
\newblock {\em Foundations of Computational Mathematics}, 16(4):1031--1068.

\bibitem[Zhang, 2019]{zhang2019cross}
Zhang, A. (2019).
\newblock Cross: Efficient low-rank tensor completion.
\newblock {\em The Annals of Statistics}, 47(2):936--964.

\bibitem[Zhou et~al., 2013]{zhou2013tensor}
Zhou, H., Li, L., and Zhu, H. (2013).
\newblock Tensor regression with applications in neuroimaging data analysis.
\newblock {\em Journal of the American Statistical Association},
  108(502):540--552.

\end{thebibliography}
